\documentclass[10pt,twocolumn,journal]{IEEEtran}

\pdfoutput=1
\usepackage{xcolor}
\usepackage{tikz}
\usepackage[most]{tcolorbox}

\newtcolorbox{OuterBracketedParagraph}{
  enhanced,
  breakable,
  sharp corners,
  boxrule=0pt,
  frame hidden,
  colback=white,
  overlay={
    \draw[line width=1pt, color=blue]
      ([xshift=0mm]frame.north east) -- 
      ([xshift=0mm]frame.south east);
    \draw[line width=1pt, color=blue]
      ([xshift=0mm]frame.north east) -- 
      ([xshift=-2mm]frame.north east);
    \draw[line width=1pt, color=blue]
      ([xshift=0mm]frame.south east) -- 
      ([xshift=-2mm]frame.south east);
  }
}
\usepackage{helvet} 
\usepackage{courier} 
\usepackage[hyphens]{url} 
\usepackage{graphicx} 
\usepackage{booktabs} 
\usepackage{amsmath} 
\DeclareMathOperator{\argmax}{argmax}
\usepackage{mathtools} 
\usepackage{amssymb} 
\usepackage{amsfonts} 
\usepackage{bm} 
\usepackage{xcolor} 
\usepackage{paralist} 
\usepackage{enumitem} 
\usepackage{sidecap} 
\usepackage{amsthm} 
\usepackage{algorithm} 
\usepackage{algpseudocode} 
\usepackage{newfloat} 
\usepackage{listings} 
\usepackage{todonotes} 
\usepackage[normalem]{ulem} 
\usepackage[caption=false,font=footnotesize]{subfig}
\usepackage[numbers]{natbib}
\usepackage{relsize}

\floatstyle{ruled}
\newfloat{listing}{tb}{lst}{}
\floatname{listing}{Listing}

\usepackage[letterpaper]{geometry}
\newtheorem{claim}{Claim}

\newtheorem{defn}{Definition}

\newtheorem*{result*}{Theorem}

\definecolor{airforceblue}{rgb}{0.36, 0.54, 0.76}
\definecolor{Lblue}{rgb}{0.36, 0.54, 1}

\newcommand{\bx}{{\mathbf x}}

\newcommand{\by}{\bm{y}}
\newcommand{\bw}{\bm{w}}

\newcommand{\bW}{\bm{w}}

\newcommand{\R}{\mathbb{R}}

\newcommand{\Myparagraph}[1]{\bigskip\textbf{#1}\;}
\newcommand{\myparagraph}[1]{\medskip\textbf{#1}\;}
\newcommand{\app}{{Appendix~}}

\IEEEoverridecommandlockouts
\title{Forget Me Not: Fighting Local Overfitting with Knowledge Fusion and Distillation}

\author{
    Uri Stern\textsuperscript{*}, Eli Corn\textsuperscript{*} and Daphna Weinshall\\
    \IEEEauthorblockA{School of Computer Science and Engineering, The Hebrew University of Jerusalem, Jerusalem 91904, Israel\\
    Email: \{ustern@gmail.com, eli.corn@mail.huji.ac.il, daphna@mail.huji.ac.il\}}
\thanks{\textsuperscript{*}Equal contribution}
}

\begin{document}

\maketitle

\begin{abstract}

Overfitting in deep neural networks occurs less frequently than expected. This is a puzzling observation, as theory predicts that greater model capacity should eventually lead to overfitting -- yet this is rarely seen in practice. But what if overfitting does occur, not globally, but in specific sub-regions of the data space? In this work, we introduce a novel score that measures the \emph{forgetting rate} of deep models on validation data, capturing what we term \emph{local overfitting}: a performance degradation confined to certain regions of the input space. We demonstrate that local overfitting can arise even without conventional overfitting, and is closely linked to the double descent phenomenon.

Building on these insights, we introduce a two-stage approach that leverages the training history of a single model to recover and retain forgotten knowledge: first, by aggregating checkpoints into an ensemble, and then by distilling it into a single model of the original size, thus enhancing performance without added inference cost.
Extensive experiments across multiple datasets, modern architectures, and training regimes validate the effectiveness of our approach. Notably, in the presence of label noise, our method -- \emph{Knowledge Fusion} followed by \emph{Knowledge Distillation} -- outperforms both the original model and independently trained ensembles, achieving a rare win-win scenario: reduced training and inference complexity.

\end{abstract}


\section{Introduction}
\label{sec:introduction}

Overfitting a training set is considered a fundamental challenge in machine learning. Theoretical analyses predict that as a model gains additional degrees of freedom, its capacity to fit a given training dataset increases. Consequently, there is a point at which the model becomes too specialized for a particular training set, leading to an increase in its generalization error. In deep learning, one would expect to see increased generalization error as the number of parameters and/or training epochs increases. Surprisingly, even vast deep neural networks with billions of parameters seldom adhere to this expectation, and overfitting as a function of epochs is almost never observed \citep{liu2022convnet}. Typically, a significant increase in the number of parameters still results in enhanced performance, or occasionally in peculiar phenomena like the double descent in test error \citep{annavarapu2021deep}, see Section~\ref{sec:overfitanddoubledescent}. Clearly, there exists a gap between our classical understanding of overfitting and the empirical results observed when training modern neural networks.

To bridge this gap, we present a fresh perspective on overfitting. Instead of solely assessing it through a decline in \emph{validation accuracy}, we propose to monitor what we term the model's \textbf{forget fraction}. This metric quantifies the portion of test data (or validation set) that the model initially classifies correctly but misclassifies as training proceeds (see illustration in Fig.~\ref{fig:illustration}). Throughout this paper we term the decline in test accuracy as \textbf{forgetting}, to emphasize that the model's ability to correctly classify portions of the data is reduced. 

In Section~\ref{sec:overfitanddoubledescent}, we investigate various benchmark datasets, observing this phenomenon even in the absence of overfitting as conventionally defined, i.e., when test accuracy increases throughout. Notably, this occurs in competitive networks despite the implementation of modern techniques to mitigate overfitting, such as data augmentation and dropout. Our empirical investigation also reveals that forgetting of patterns occurs alongside the learning of new patterns in the training set, explaining why the traditional definition of overfitting fails to capture this phenomenon.

\begin{figure}[!t]
    \centering
    \includegraphics[width=0.4\linewidth]{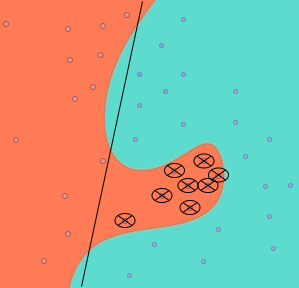} \hspace{0.35cm}
    \includegraphics[width=0.4\linewidth]{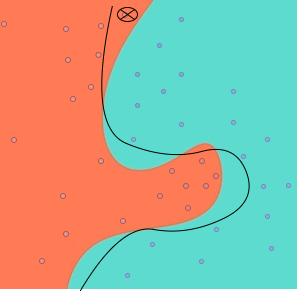}
    \caption{Local overfitting and forgetting in a binary problem, where blue and orange denote the different classes, and circles mark the validation set. The initial (left) and final (right) separators of a hypothetical learning method are shown, where $\otimes$ marks prediction errors. Clearly the final classifier has a smaller generalization error, but now one point at the top is 'forgotten'.}
    \label{fig:illustration}
    \vspace{-0.5cm}
\end{figure}

\begin{table}[t]
\caption{Summary of Method Complexity and Performance}
\label{tab:method_summary}
\begin{center}
\footnotesize

\resizebox{\columnwidth}{!}{%

\begin{tabular}{l|cc|cc}
\toprule
& \multicolumn{2}{c|}{Complexity} & \multicolumn{2}{c}{Ranking} \\
\textbf{Method} & \textbf{Training } & \textbf{Inference } & \textbf{Label Noise} & \textbf{Clean Data} \\
\midrule
Single Model                    & Low  & Low  & 5 & 5 \\
Independent Ensemble            & High & High & 2-3 & \textbf{1} \\
Distilled Independent Ensemble  & High & Low  & 2-3 & 3-4 \\
Knowledge Fusion (KF)           & Low  & High & 4 & 3-4 \\
\textbf{Distilled KF Ensemble} & \textbf{Low}  & \textbf{Low}  & \textbf{1} & 2 \\
\bottomrule

\end{tabular}
}
\end{center}

\footnotesize{\vspace{0.25cm}Each method is ranked from 1 (best) to 5 (worst) based on its average performance across multiple datasets (see Table~\ref{table:IndependentEnsembleKD}). Note that distilled KF offers the best performance-efficiency trade-off, ranking first with label noise and  second on clean data, while maintaining low complexity.
}
\vspace{-0.25cm}
\end{table}

Based on the empirical observations reported in Section~\ref{sec:overfitanddoubledescent}, we propose in Sections~\ref{sec:method}-\ref{sec:KnowledgeDistillation} a two phase method that can effectively reduce the forgetting of test data, thus improving the final accuracy and reduce overfitting, while maintaining inference efficiency similar to a single model. 

In phase one (Section~\ref{sec:method}), we introduce a new prediction method, \emph{Knowledge Fusion}, that combines knowledge gained in different stages of training. This method delivers a weighted average of the class probability output between the final model and a set of checkpoints of the model from mid-training. The checkpoints and their respective weights are selected iteratively using a validation dataset and our forget metric. The purpose of this phase is two-fold: First, it improves upon the original model, serving as another strong indication that models indeed forget useful knowledge in the late stages of training. Second, it provides a proof-of-concept that this lost knowledge can be recovered.

In phase two (Section~\ref{sec:KnowledgeDistillation}), we address the increased inference complexity introduced by phase 1. Since phase 1 effectively forms an ensemble by aggregating predictions from multiple checkpoints, its memory and compute demands during inference are significantly higher than those of a single model. This overhead can hinder deployment in resource-constrained environments. To solve this, in phase 2 we \emph{condense} the ensemble into a single, compact model through two approaches: (i) \textbf{Knowledge Distillation (KD)}, where a student network is trained to mimic the outputs of the KF ensemble; and (ii) \textbf{Weight Averaging}, a simpler method that directly averages the weights of selected checkpoints. These approaches aim to retain the performance benefits of KF while restoring the efficiency of inference by a single model.

Empirical results, presented in Section~\ref{sec:empirical}, validate the effectiveness of our approach across multiple image classification datasets, with and without label noise. Using various network architectures, including modern networks on ImageNet, the results demonstrate that our method is universally beneficial, consistently improving upon the original model. To ensure a fair comparison, we compare against the original model, which also leverages a validation set to select the best performing checkpoint for early stopping, since this aligns with our method’s use of a validation set. 

When compared to alternative methods that leverage the network's training history, our approach delivers comparable or superior performance, while retaining lower inference complexity. When compared to alternative ensemble methods of higher training \textbf{and} inference complexity, the distilled condensed model retains (with a few exceptions), and in some cases surpasses, the performance of alternative methods. This provides a win-win scenario: reduced training and inference costs without sacrificing generalization. 

Table~\ref{tab:method_summary} summarizes these results by ranking the aforementioned core methods across two criteria: performance on noisy and clean datasets, and their training/inference complexity. Rankings (1 = best, 5 = worst) are computed by first ranking the methods separately on each dataset and then averaging those ranks, based on accuracy scores from Table~\ref{table:IndependentEnsembleKD} of Section~\ref{sec:empirical}. Notably, the \emph{distilled KF ensemble} achieves the best performance on noisy datasets and strong performance on clean datasets, while maintaining \textbf{low} complexity in both training and inference. In contrast, while the \emph{independent ensemble} performs best on clean data, it incurs significantly higher computational costs. These results underscore the distilled KF ensemble as offering the most favorable trade-off between performance and efficiency across both clean and noisy data settings.

The theoretical investigation of forgotten knowledge is presented in Section~\ref{sec:TheoryForgottenKnowledge}, which adopts the framework of over-parameterized deep linear networks. In \citep{stern2025overfit} it is shown that the theoretical analysis of forgotten points correlates significantly with empirical results, offering insights into local overfitting and the phenomenon of forgotten knowledge.

\subsection*{Our main contributions}
\begin{inparaenum}[(i)]
    \item A novel perspective on overfitting, formalizing the notion of \emph{local overfitting}.
    \item Empirical evidence that overfitting occurs \textbf{locally} even without a decrease in overall generalization.
    \item An effective method to reduce local overfitting, and its empirical validation.
    This method retains the low training and inference costs of a single model, while preserving superior performance.
\end{inparaenum}


\section{Related Work}
\label{sec:RelatedWork}

\subsection*{Study of forgetting in prior work} Most existing studies examine the forgetting of training data, where certain training points are initially memorized but later forgotten. This typically occurs when the network cannot fully memorize the training set. In contrast, our work focuses on the \textbf{forgetting of validation points}, which arises when the network successfully memorizes the entire training set. Building on \citep{arpit2017closer}, who show that networks first learn "simple" patterns before transitioning to memorizing noisy data, we analyze the later stages of learning, particularly in the context of the double descent phenomenon. Another related but distinct phenomenon is "catastrophic forgetting" \citep{mccloskey1989catastrophic}, which occurs in \emph{continual learning} settings where the training data evolves over time—unlike the static training scenario considered here.
 
\subsection*{Ensemble learning} Ensemble learning has been studied extensively see 
\citep{yang2023survey}. Our methodology is rooted in "implicit ensemble learning", in which only a single network is trained in a way that "mimics" ensemble learning \citep{srivastava2014dropout}. Utilizing checkpoints from the training history as a 'cost-effective' ensemble has also been considered. This was achieved by either considering the last epochs and averaging their probability outputs \citep{xie2013horizontal}, or by employing exponential moving average (EMA) on all the weights throughout training \citep{polyak1992acceleration}. The latter method does not always succeed in reducing overfitting, as discussed in \citep{izmailov2018averaging}.

Several methods 
modify the training protocol to converge to multiple local minima, which are then combined into an ensemble classifier. While these approaches show promise \citep{noppitak2022dropcyclic}, they add complexity to training and may even hurt performance \citep{guo2023stochastic}. Our comparisons (see Table~\ref{table:specialmethods}) demonstrate that our simpler method either matches or outperforms these techniques in all studied cases.

\subsection*{Studies of overfitting and double descent} 

Double descent with respect to model size has been studied empirically in  \citep{belkin2019reconciling, nakkiran2021deep}, while epoch-wise double descent (which is the phenomenon analyzed here) was studied in \citep{stephenson2021and, heckel2020early}. These studies analyzed when and how epoch-wise double descent occurs, specifically in data with label noise, and explored ways to avoid it (sometimes at the cost of reduced generalization). Essentially, our research identifies a similar phenomenon in data without label noise. It is complementary to the study of "benign overfitting", e.g., the fact that models can achieve perfect fit to the train data while still obtaining good performance over the test data.

\subsection*{Knowledge distillation} 
\label{sec:related-work-distillation}

Knowledge Distillation (KD) \citep{hinton2015distillingknowledgeneuralnetwork} is a technique in which a smaller "student" model is trained to replicate the behavior of a larger "teacher" model by learning from its output predictions. In traditional KD, the teacher model provides ``soft targets'' for the student to learn from, instead of relying solely on hard labels. This enables the student to capture a more nuanced understanding of the data distribution, which often results in improved performance.

Recall that ensemble learning can significantly increase computational costs at inference time. Knowledge distillation—where the ensemble acts as the teacher and a single model serves as the student—offers a way to reduce these costs, though often with some loss in performance. Interestingly, in certain settings such as those with high label noise, the distilled model can even outperform the ensemble itself \citep{jeong2024understanding,stern2024united}.

Self-distillation is a technique in which a model learns from its own prior predictions, enhancing performance without the need for a separate teacher model. This technique has been shown to enhance the original model's performance, particularly in noisy environments \citep{allen2020towards, pham2022revisiting, das2023understanding}, with several theoretical explanations proposed for its effectiveness.

Somewhat counter-intuitively, recent studies have shown that higher-performing models do not necessarily make better teachers in knowledge distillation. Accordingly, \citep{wang2022efficient} show that an intermediate checkpoint from a teacher's training trajectory can serve as a more effective teacher than the final, fully trained model. 
They also demonstrate that ensembles built from a random sample of checkpoints, while under-performing independently trained ensembles, are better able to guide student models. \citep{wang2022learn} expands on this idea by proposing to use a self-attention mechanism to adaptively weight intermediate models during distillation.

Some of our empirical findings mirror these insights, particularly the superior performance of ensembles composed of intermediate checkpoints over independently trained models in knowledge distillation. While \citep{wang2022learn} primarily addresses the problem of how to weight each randomly sampled checkpoint, we focus on a complementary challenge - the effective selection of checkpoints - guided by the notion of forgotten knowledge. Their proposed adaptive weighting technique further adds complexity to the training. Consequently, and given the absence of publicly available code, our empirical comparisons are limited to a classical checkpoint sampling method (see Table~\ref{table:ema}). Nonetheless, we note that the adaptive weighting of checkpoints can be incorporated into our approach as a complementary enhancement.

\section{Overfitting Revisited}
\label{sec:overfitanddoubledescent}

\begin{figure*}[t]
    \centering
    \subfloat[Cifar100, original data\label{subfig:DDc100}]
    {
        \includegraphics[width=0.24\textwidth]{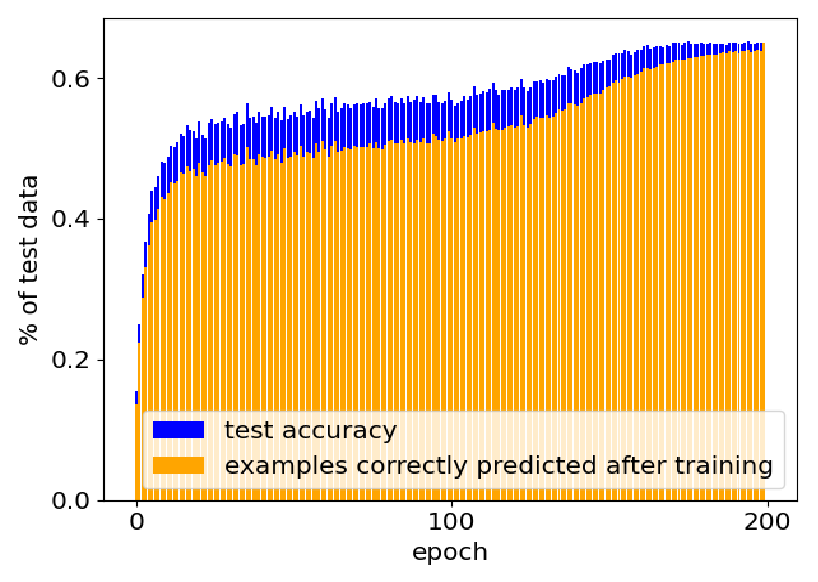}
    }\hspace{-0.25cm}
    \subfloat[Cifar100, 20\% sym noise\label{subfig:DDc100sym20}]{
        \includegraphics[width=0.24\textwidth]{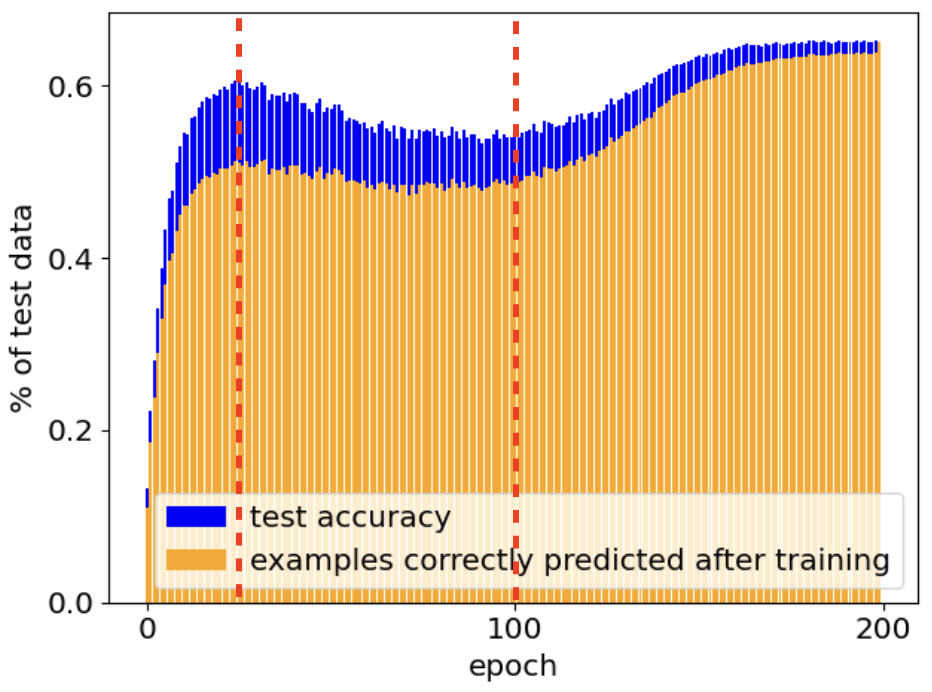}
    }\hspace{-0.25cm}
    \subfloat[Cifar100, 20\% sym noise\label{subfig:c100sym20testacc}]{
        \includegraphics[width=0.24\textwidth]{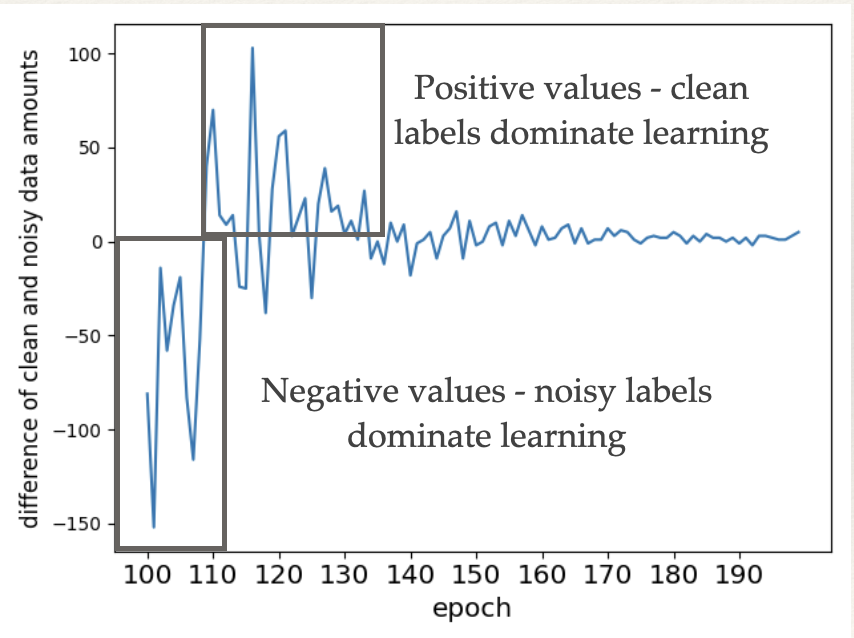}
    }\hspace{-0.25cm}
    \subfloat[TinyImageNet, 20\% sym noise\label{subfig:logitsheatmapc100}]{
        \includegraphics[width=0.24\textwidth]{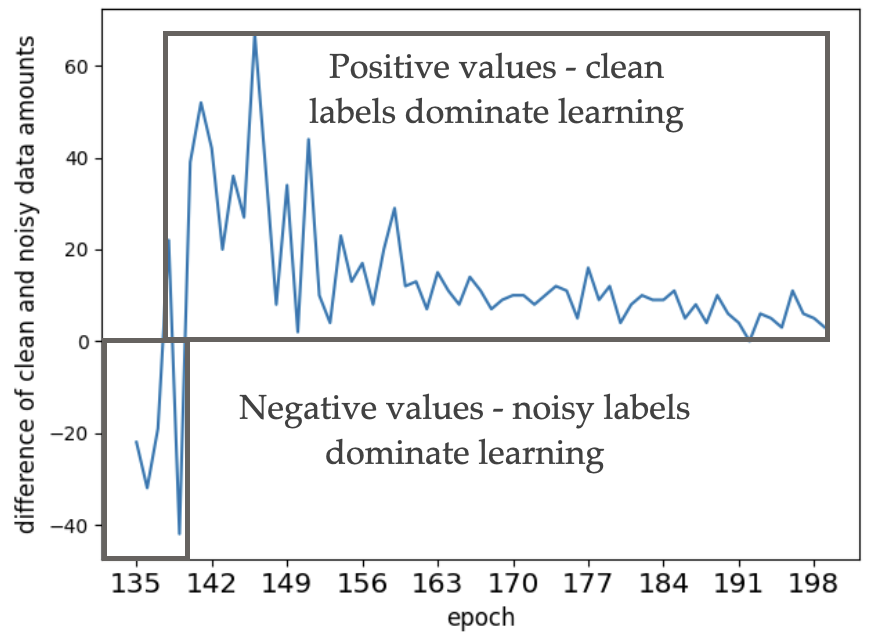}
    }
    \caption[heatmap]{(a)--(b): Blue denotes test accuracy. Among those correctly recognized in each epoch $e$, orange denotes the fraction that remains correctly recognized at the end. The test accuracy (the blue curve) in (b) shows a clear double ascent of accuracy, which is much less pronounced in the orange curve. During the decrease in test accuracy—the range of epochs between the first and second dashed red vertical lines—the large gap between the blue and orange plots indicates the fraction of test data that has been correctly learned in the first ascent and then forgotten, without ever being re-learned in the later recovery period of the second ascent. This gap also appears in (a), demonstrating that forgetting occurs even without double ascent. (c)--(d): The difference between the number of clean and noisy datapoints at each epoch during the second ascent of test accuracy (the epochs after the second dashed red vertical line), counting datapoints with large loss only. Positive (negative) values indicate that clean (noisy) datapoints are more dominant in the corresponding epoch.}
    \label{fig:doubledescent}
\end{figure*}

The textbook definition of overfitting entails the co-occurrence of increasing train accuracy and decreasing generalization. Let $acc(e,S)$ denote the accuracy over set $S$ in epoch $e$, where $E$ is the total number of epochs, and $T$ is the test dataset. Throughout this paper, the terms "test set" and "validation set" are used interchangeably. When approximating generalization by test accuracy, overfitting is said to occur at epoch $e$ if $acc(e,T) \geq acc(E,T)$.  

We begin by investigating the hypothesis that portions of the test data $T$ may be forgotten by the network during training. 
When we examine  the 'epoch-wise double descent', which frequently occurs during training on datasets with significant label noise, we indeed observe that a notable forgetting of the test data coincides with the memorization of noisy labels. Here, forgetting serves as an objective indicator of overfitting. When we further examine the training of modern networks on standard datasets (devoid of label noise), where overfitting (as traditionally defined) is absent, we discover a similar phenomenon (though of weaker magnitude): \emph{the networks still appear to forget certain sub-regions of the test population}. This observation, we assert, signifies a significant and more subtle form of overfitting in deep learning.

\subsection*{Local overfitting} Let $M_{e}$ denote the subset of the test data \emph{mislabeled} by the network at some epoch $e$. We define below two scores $L_e$ and $F_e$:
\begin{equation}
\label{eq:forget}
    F_e = \frac{acc(e,M_{E})\cdot|M_{E}|}{|T|}, ~
    L_e = \frac{acc(E,M_{e})\cdot|M_{e}|}{|T|}
\end{equation}
The forget fraction $F_e$ represents the fraction of test points correctly classified at epoch $e$ but misclassified by the final model. $L_e$ represents the fraction of test points misclassified at epoch $e$ but correctly classified by the final model. The relationship $acc(E, T) = acc(e, T)+L_e-F_e$ follows\footnote{$acc(E,T) - L_{e} = acc(e,T) - F_{e}$ is the fraction of test points correctly classified in both $e$ and $E$.}. In line with the classical definition of overfitting, if $L_{e} < F_{e}$, overfitting occurs since $acc(E, T) < acc(e, T)$.

But what if $L_e \geq F_e~\forall e$? By its classical definition, overfitting does not occur since the test accuracy increases continuously. Nevertheless, there may still be local overfitting as defined above, since $F_e>0$ indicates that data has been forgotten even if $L_e \geq F_e$. 

\subsection*{Reflections on the epoch-wise double descent} Epoch-wise double descent (see Fig.~\ref{fig:doubledescent}) is an empirical observation \citep{belkin2019reconciling}, which shows that neural networks can improve their performance even after overfitting, thus causing double descent in test error during training (note that we show the corresponding \emph{double-ascent in test accuracy}). This phenomenon is characteristic of learning from data with label noise, and is strongly related to overfitting since the dip in test accuracy co-occurs with the memorization of noisy labels. 

We examine the behavior of score $F_e$ in this context and make a novel observation: when we focus on the fraction of data correctly classified by the network during the second rise in test accuracy, we observe that the data newly memorized during these epochs often differs from the data forgotten during the overfitting phase (the dip in accuracy). In fact, most of this data has been previously misclassified (see Fig.~\ref{subfig:DDc100sym20}). Figs.~\ref{subfig:c100sym20testacc}-\ref{subfig:logitsheatmapc100} further illustrate that during the later stages of training on data with label noise, the majority of the data being memorized is, in fact, data with clean labels, which explains the second increase in test accuracy. It thus appears that epoch-wise double descent is caused by the simultaneous learning of general (but hard to learn) patterns from clean data, and irrelevant features of noisy data.

\subsection*{Forgetting in the absence of label noise} When training deep networks on visual benchmark datasets without added label noise, overfitting rarely occurs. In contrast, we observe that local overfitting, as captured by our new score $F_e$, commonly occurs, as demonstrated in Fig.~\ref{subfig:DDc100}.

\begin{figure}[htb]
    \centering
    \captionsetup[subfigure]{font=footnotesize} 
    
    \subfloat[\resizebox{.25\linewidth}{!}{ImageNet}\label{subfig:modelsizeforget}]{
        \includegraphics[width=0.55\linewidth]{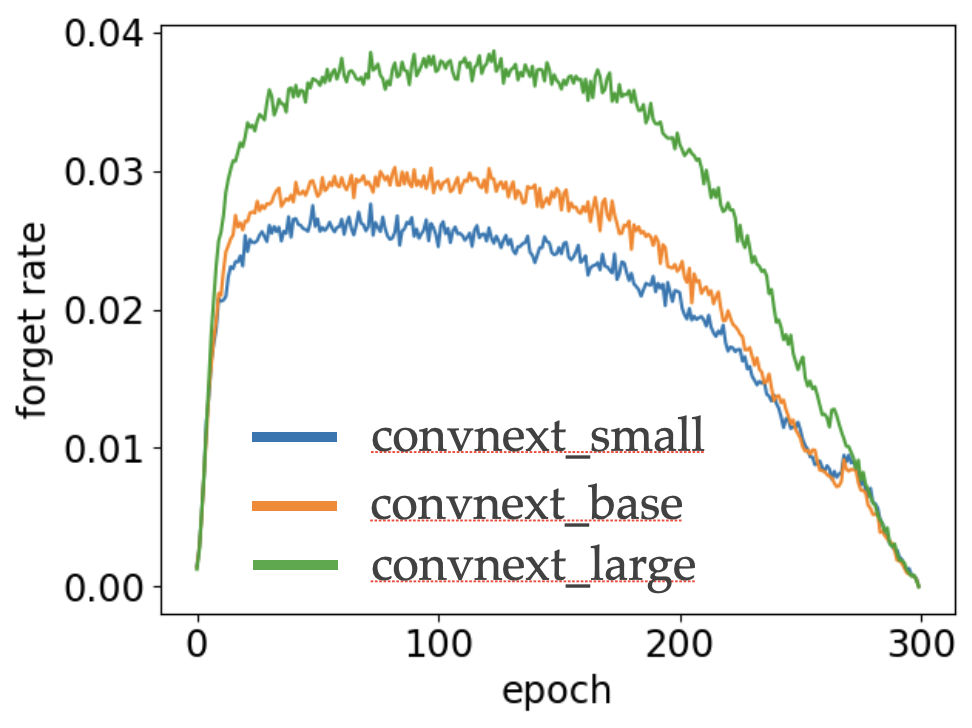}
        \vspace{-0.5cm}
    }
    \subfloat[\resizebox{.6\linewidth}{!}{TinyImageNet, ResNet18}\label{subfig:forgettimetimg}]{
        \includegraphics[width=0.42\linewidth]{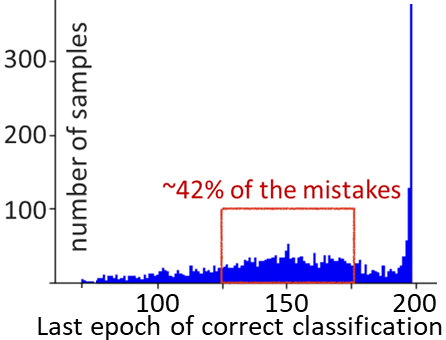}
       \vspace{0.5cm}
    }
    \caption[forget]{(a) The $F_e$ score (\ref{eq:forget}) of ConvNeXt trained on ImageNet, 3 network sizes: small $\rightarrow$ blue, base $\rightarrow$ orange, and large $\rightarrow$ green. Accuracy remained consistent across all network sizes, while it is evident that $F_e$ increases with the network size. (b)Within the set of wrongly classified test points after training, we show the last epoch in which an example was classified correctly.}
    \label{fig:forget}
\end{figure}

To investigate this phenomenon, we trained various neural networks (ConvNets: Resnet, ConvNeXt; Visual transformers: ViT, MaxViT) on various datasets (CIFAR-100, TinyImageNet, ImageNet) using a variety of optimizers (SGD, AdamW) and learning rate schedulers (cosine annealing, steplr). Results on ImageNet are shown in Fig.~\ref{subfig:modelsizeforget}; the results of additional experiments on other datasets and configurations can be found in Fig.~\ref{subfig:trainsizeforget} of \app\ref{app:moreforget} in the Supplementary material. 

Across all settings, we observe that networks consistently forget a subset of the data during training, in a manner similar to what is observed with label noise, even when test accuracy continues to improve. Fig.~\ref{subfig:forgettimetimg} (see additional results and analysis in Figs.~\ref{subfig:cifar100sumcorrect}-\ref{subfig:forgettimecifar100}) demonstrate that this effect is not simply due to random fluctuations: many test examples that are incorrectly classified post training have been correctly classified during much of the training. These results are connected to overfitting in Fig.~\ref{subfig:modelsizeforget}: when investigating larger models and/or relatively small amounts of train data, which are scenarios that are expected to increase overfitting based on theoretical considerations, the \emph{forget fraction} $F_e$ is larger. 

\Myparagraph{Summary} We see that neural networks can, and often will, "forget" significant portions of the test population as their training proceeds. In a sense, the networks \emph{are} overfitting, but this only occurs at some limited sub-regions of the feature space. The reason this failing is not captured by the classical definition of overfitting is that the networks continue to learn new general patterns simultaneously. In Section~\ref{sec:method} we discuss \emph{how we can harness this observation to improve the network's performance}.

\section{Knowledge Fusion (KF)}
\label{sec:method}

In Section~\ref{sec:overfitanddoubledescent} we showed that neural networks often achieve better performance in mid-training on a subset of the test data, even when the test accuracy is monotonically increasing with training epochs. In this section, we introduce an approach to integrate the knowledge obtained in both mid- and post-training epochs during inference time, to improve performance. To this end we must determine: \begin{inparaenum}[(i)] \item  which versions of the model to use; \item how to combine them with the post-training model; and \item how to weigh each model in the final ensemble. \end{inparaenum}  

We call our method \textbf{KnowledgeFusion} (KF). Below, we outline the core components of the approach and provide further implementation details in the following subsections.

\paragraph{Choosing an early epoch of the network} Given a set of epochs $\{1,\ldots,E\}$ and corresponding forget rates $\{F_e\}_e$, we first single out the model $n_A$ obtained at epoch $A = \argmax_{e \in \{1,...,E\}}F_e$. This epoch is most likely to correctly fix mistakes of the model on "forgotten" test data.

\paragraph{Combining the predictors} Next, using validation data we determine the relative weights of the two models—the final model $n_E$, and the intermediate model $n_A$ with maximal forget fraction. Since the accuracy of $n_E$ is typically much higher than $n_A$, and in order not to harm the ensemble's performance, we expect to assign $n_E$ a higher weight. 

\paragraph{Improving robustness} To improve our method's robustness to the choice of epoch $A$, we use a window of epochs around $A$, denoted by $\{n_{A-w},...,n_A,..., n_{A+w}\}$. The vectors of probabilities computed by each checkpoint are averaged before forming an ensemble with $n_E$. In our experiments we use a fixed window $w=1$, achieving close to optimal results, as verified in the ablation study Section~\ref{sec:empirical}.

\paragraph{Iterative selection of models} As we now have a new predictor, we can find another alternative predictor from the training history that maximizes accuracy on the data misclassified by the new predictor, in order to combine their knowledge as described. This can be repeated iteratively, until no further improvement is achieved. 

\paragraph{Choosing hyper-parameters} In order to compute $F_e$ and assign optimal model weights and window size, we use a validation set, which is a part of the labeled data not shown to the model during initial training. This is done \textbf{post training} as it has no influence over the training process, and thus \emph{doesn't incur additional training cost}. Following common practice, we show in the ablation study below that after optimizing these hyper-parameters, it is possible to retrain the model on the complete training set while maintaining the same hyper-parameters. The performance of KF thus trained is superior to alternative methods trained on the same data. 

\subsection{Hyper-parameter Calculation}

Algorithm~\ref{alg:val} outlines the hyper-parameter search phase of our method, Knowledge Fusion, which is performed once after training. It is responsible for selecting the most relevant epochs and their associated weights.

\begin{algorithm}[htb]
   \caption{KF - hyper-parameter calculation}
   \begin{algorithmic}
      \State {\bfseries Input:} all past checkpoints during training of the neural network \{ $n_0$,...,$n_E$ \}, $w$, and validation data $V$
      \State {\bfseries Output:} list of alternative epochs and their weights
      \State $class\_probs \gets$ \textbf{get\_class\_probs}(\{ $n_0$,...,$n_E$ \}, $V$)
      \State $prob \gets class\_probs[E]$
      \State explore = \{ $n_0$,...,$n_E$ \}
      \State Alternative\_epochs = \{\}
      \State epsilons = \{\}
      \While{explore is not empty}
         \State $F \gets$ \textbf{calc\_forget\_per\_epoch}($prob$, $class\_probs$)
         \State $alt\_epoch \gets$ \textbf{argmax}($F[explore]$)
         \State Alternative\_epochs.\textbf{append(alt\_epoch)}
         \State explore.\textbf{remove(alt\_epoch - 1, alt\_epoch, alt\_epoch + 1)}
         \For{$\varepsilon \in \{0, 0.01, ..., 1\}$}
            \State $prob_A \gets \mathbf{mean}(class\_probs[A_i-w:A_i+w])$
            \State $combined\_prob \gets \varepsilon \cdot prob_A + (1-\varepsilon) \cdot prob$
            \If{\textbf{validation\_acc}(combined\_prob) $\geq$ \textbf{validation\_acc}(prob)}
               \State $best\_prob \gets combined\_prob$
               \State $best\_epsilon \gets combined\_prob$
            \EndIf
         \EndFor
         \State $prob \gets best\_prob$
         \State epsilons.\textbf{append(argmax(best\_epsilon))}
      \EndWhile
      \State \textbf{Return} Alternative\_epochs, epsilons
   \end{algorithmic}
   \label{alg:val}
\end{algorithm}

\subsection{Prediction}

Algorithm~\ref{alg:test} shows the prediction step of Knowledge Fusion (KF). Using the epochs and weights computed in Algorithm~\ref{alg:val}, it combines the predictions of selected checkpoints to produce the final output for a test sample. As in Algorithm~\ref{alg:val}, the \textbf{get\_class\_probs} function is utilized to obtain the class probabilities for a given example and a list of predictors.

\begin{algorithm}[htbp]
\small
   \caption{Knowledge Fusion (KF) Prediction}
   \begin{algorithmic}
      \State {\bfseries Input:} Checkpoints of trained model \{ $n_0$,...,$n_E$ \}, selected epochs \{ $A_1$,...,$A_k$ \}, associated weights \{ $\varepsilon_1$,...,$\varepsilon_k$ \}, $w$, test-pt $x$
      \State {\bfseries Output:} prediction for $x$
      \State $prob \gets \textbf{get\_class\_probs}[E]$
      \For{$i \gets 1$ \textbf{to} $k$}
         \State $prob_A \gets \textbf{mean}(\textbf{get\_class\_probs}[A_i-w:A_i+w])$
         \State $prob \gets \varepsilon_i \cdot prob_A + (1-\varepsilon_i) \cdot prob$
      \EndFor
      \State $prediction \gets \textbf{argmax}(prob)$
      \State \textbf{Return} $prediction$
   \end{algorithmic}
   \label{alg:test}
\end{algorithm}

\subsection{Discussion}
\label{subsec:discussion}

Empirical results (see Section~\ref{sec:empirical}) show that KF consistently performs on par with, or better than, strong baselines. This holds across a range of architectures, datasets, and training conditions.

However, a key limitation of our method lies in its computational cost at inference. Since predictions involve aggregating outputs from multiple model checkpoints, runtime and memory usage scale linearly with the number of component in the ensemble. This cost can be prohibitive in resource-constrained deployment settings, making the method less practical despite its accuracy benefits.

To address this, the next section introduces strategies for condensing the KF ensemble into a single, compact model. These strategies aim to preserve the performance advantages of the ensemble while enabling efficient inference, thereby broadening the practical applicability of KF.

\section{Condensing Knowledge Fusion}
\label{sec:KnowledgeDistillation}

The KF method is designed to enhance model accuracy while successfully mitigating local overfitting, as demonstrated in Section~\ref{sec:empirical}. It can be applied as a post-processing step to any trained model with only a minor increase in training costs. However, because KF produces an ensemble classifier, its adoption results in a substantial increase in inference time.

To address this issue, we explore an additional post-processing step that compresses the ensemble classifier into a single model, thereby reducing inference costs to a level comparable with the baseline. Below, we discuss two such methods that entail a trade-off between implementation complexity and model accuracy, as demonstrated in Section~\ref{sec:empirical}.

\subsection{Method 1: Distilled KF}

Our first approach leverages Knowledge Distillation (KD). This technique involves training a smaller, more efficient model (the student) to replicate the behavior of a larger, more complex model (the teacher). In the context of Knowledge Fusion, the teacher model is represented by the KF ensemble, while the student is a single model distilled from this ensemble. The goal is to combat local overfitting by leveraging the combined knowledge that KF has accumulated from earlier stages of training, while fully eliminating the memory overhead and increased inference time associated with maintaining a complete ensemble.

In our experiments presented later in this section, the distilled network was designed to match the size of the original single network. This ensures a fair comparison in terms of resource usage during inference and similar training times. We conducted multiple experiments across various settings, including different noise levels and datasets, all of which demonstrated the effectiveness of this approach. Our results indicate that applying Knowledge Distillation to Knowledge Fusion not only enhances computational efficiency but also preserves the high accuracy of the ensemble, confirming its practicality for large-scale tasks.

\myparagraph{Brief Overview of Knowledge Distillation}
We now describe the standard implementation of Knowledge Distillation (KD), introduced earlier in Section~\ref{sec:related-work-distillation}. KD typically relies on two loss terms: a soft target loss that compares teacher and student predictions, and a true label loss computed using ground-truth labels.

The soft target loss is computed by comparing the logits of the teacher model with the logits of the student model. Both sets of logits are scaled by a temperature parameter \( T \), which smooths the logits and generates softened probability distributions. This allows the student to learn from the teacher's more informative, probabilistic outputs.

Neural networks typically produce class probabilities by applying the \textit{softmax} function to the output logits. Given a vector of logits \( z \in \mathbb{R}^{C} \), the temperature-scaled softmax function converts each logit \( z_c \) into a probability by comparing it with the other logits:
\[
\text{softmax}\left(\frac{z}{T}\right)_c = \frac{e^{z_c / T}}{\sum_{j=1}^{C} e^{z_j / T}}.
\]
This results in a smooth probability distribution over the \( C \) classes, with higher temperature \( T \) producing a more uniform distribution.

Given teacher logits \( \mathbf{z}_\text{T} \) and student logits \( \mathbf{z}_\text{S} \), we define their soft probability vectors as
\[
\hat{\mathbf{p}}_\text{T} = \text{softmax}\left( \frac{\mathbf{z}_\text{T}}{T} \right), \quad
\hat{\mathbf{p}}_\text{S} = \text{softmax}\left( \frac{\mathbf{z}_\text{S}}{T} \right).
\]
The soft target loss is then defined as the Kullback-Leibler (KL) divergence between the teacher and student distributions:
\[
\mathcal{L}_\text{soft} = T^2 \cdot \text{KL} \left( \hat{\mathbf{p}}_\text{T} \,\|\, \hat{\mathbf{p}}_\text{S} \right).
\]
The factor of \( T^2 \) ensures that the gradient magnitudes remain comparable across different values of \( T \), as suggested by \citep{hinton2015distillingknowledgeneuralnetwork}.


Beyond the soft target loss, there is also the student's original loss, denoted as \( \mathcal{L}_{\text{label}} \), which represents the \textbf{true label loss}. In our case, it is the cross-entropy loss computed using the student model’s logits and the true labels \( \mathbf{y} \):
\[
\mathcal{L}_{\text{label}} = \text{cross\_entropy}(\mathbf{z}_{\text{S}}, \mathbf{y})
\]

The total loss is a weighted sum of the soft target loss and the true label loss with hyperparameter $\alpha\in[0,1]$:
\begin{equation}
\mathcal{L}_{\text{total}} = \alpha \cdot \mathcal{L}_{\text{soft}} + (1 - \alpha) \cdot \mathcal{L}_{\text{label}}
\end{equation}

This formulation allows the student model to retain knowledge from the teacher model (via soft targets), while ensuring that it also performs well on the ground-truth labels.

\subsection{Method 2: Averaged KF}
\label{subsec:SimpWeightAveraging}

As an alternative to applying Knowledge Distillation to our Knowledge Fusion method, and as a baseline to assess the necessity of distillation, we explored a simpler approach to reducing the complexity of the KF ensemble. This method directly combines the weights of the ensemble models into a single model. 

The procedure mirrors the Knowledge Fusion (KF) algorithm, where we iteratively select model checkpoints with high forgetting scores and incorporate them into the current model. However, rather than creating an ensemble of models, we average the network weights of the selected models to produce a single set of weights. This method allows for predictions with only a single forward pass, thus reducing the computational and memory costs compared to an ensemble.

As in the ensemble approach, each model's contribution to the final mechanism is weighted according to its performance. The optimal weight for each model is determined by testing a range of potential values on the validation set and selecting the one that yields the best performance.

\section{Empirical Evaluation}
\label{sec:empirical}

In this section, we evaluate the performance of our method, with and without phase 2, against the original predictor (i.e., the early-stopped trained network) and other baselines. Full implementation details are provided in the supplementary material, \app\ref{app:implementationdetails}. 

\subsection{Baselines} 

KF incurs the training cost of a single model, and thus, following the methodology of \citep{huang2017snapshot}, we begin by comparing our method to methods that require the same amount of training time. The first group of baselines includes methods that do not alter the training process:
\begin{itemize}
    \item \textbf{Single network}: The original network, after training, which in this case refers to the early-stopped version of the model, where the validation set is used to choose the best stopping point. 
    \item \textbf{Horizontal ensemble} \citep{xie2013horizontal}: This method uses a set of epochs at the end of the training, and delivers their average probability outputs (with the same number of checkpoints as we do). 
    \item \textbf{Fixed jumps}: This baseline was used in \citep{huang2017snapshot}, where several checkpoints of the network, equally spaced through time, are taken as an ensemble.
\end{itemize} 

\begin{table}[b!]
\caption{Comparison on Clean Datasets}
\label{table:regularnetworksKD}
\begin{center}
\footnotesize
\begin{tabular}{l|c||c}
\multicolumn{1}{l|}{Method/\textbf{Dataset}} & \textbf{CIFAR-100} & \textbf{TinyImageNet} \\
\multicolumn{1}{r|}{architecture} & Resnet18 & Resnet18 \\
\toprule
\emph{baseline network} & 
$78.85 \pm .32$ & $65.30 \pm .36$ \\
\hline
\noalign{\vskip 0.5ex}
\multicolumn{3}{l}{\textbf{Ensembles}} \\
\emph{horizontal} & 
$78.42 \pm .17$ & $65.61 \pm .31$ \\
\emph{fixed jumps} & 
$78.21 \pm .46$ & $67.41 \pm .21$ \\
\emph{KF} & 
$79.36 \pm .17$ & $68.33 \pm .17$ \\
\hline
\noalign{\vskip 0.5ex}
\multicolumn{3}{l}{\textbf{Condensed Models}} \\
\emph{Averaged KF} & 
$78.48 \pm .26$ & $65.17 \pm .23$ \\
\emph{Distilled KF} & 
$\mathbf{80.29 \pm .16}$ & $\mathbf{69.96 \pm .15}$ \\

\bottomrule
\emph{improvement} & 
$1.44 \pm .36$ & $4.66 \pm .39$ \\
\bottomrule

\end{tabular}
\end{center}

\footnotesize{\vspace{0.25cm}Mean (over random validation/test splits) test accuracy (in percent) and standard error on image classification datasets, comparing our method and
baselines described in the text. The last row shows the improvement of the best performer over the baseline network. Bolded values indicate the best performer in each column.}
\vspace{-0.25cm}
\end{table}

The second group includes methods that \emph{alter} the training protocol. While this is not a directly comparable set of methods, as they focus on a complementary way to improve performance, we report their results in order to further validate the usefulness of our method. This group includes:

\begin{itemize}
\item \textbf{Snapshot ensemble} \citep{huang2017snapshot}: The network is trained in multiple "cycles," each ending with a sharp learning rate increase to escape local minima and promote convergence to diverse solutions for ensembling.
\item \textbf{Stochastic Weight Averaging} (SWA) \citep{izmailov2018averaging}: The network is first trained conventionally, then continuing with a constant or cyclic learning rate to reach multiple local minima, whose weights are averaged. For fair comparison, we adopt their setup: using our strategy for 75\% of the epochs, followed by SWA for the remaining 25\%.
\item \textbf{Fast Geometric Ensembling} (FGE) \citep{garipov2018loss}: Like SWA, where this method averages the output probabilities of diverse models rather than their weights. Computational budgets are matched as described above.

\end{itemize}


\begin{table*}[thb!]
\caption{Comparison on Noisy Datasets}
\label{table:noisyKD}
\begin{center}
\scriptsize
\begin{tabular}{l|c||c|c|c|c||c|c}
\multicolumn{1}{l|}{Method/\textbf{Dataset}} 
  & \multicolumn{1}{c||}{\textbf{CIFAR-100N}} 
  & \multicolumn{4}{c||}{\textbf{CIFAR-100}} 
  & \multicolumn{2}{c}{\textbf{TinyImageNet}} \\
\multicolumn{1}{r|}{architecture / \% noise} 
  & 40\,\% 
  & 20\,\% asym & 40\,\% asym & 20\,\% sym & 40\,\% sym 
  & 20\,\% sym & 40\,\% sym \\
\toprule
\emph{baseline network} 
  & $54.53 \pm .44$ 
  & $67.07 \pm .42$ & $49.95 \pm .53$ & $65.16 \pm .25$ & $59.03 \pm .68$ 
  & $56.25 \pm .47$ & $49.83 \pm .3$ \\
\hline
\noalign{\vskip 0.5ex}
\multicolumn{3}{l}{\textbf{Ensembles}} \\
\emph{fixed jumps} 
  & $60.38 \pm .57$ 
  & $71.46 \pm 1.91$ & $59.9 \pm .6$ & $72.8 \pm .1$ & $66.5 \pm .1$ 
  & $60 \pm .8$ & $54.16 \pm .3$ \\
\emph{horizontal}  
  & $55.14 \pm .57$ 
  & $73.04 \pm .47$ & $58.5 \pm .1$ & $71.1 \pm .38$ & $65.2 \pm .1$ 
  & $57.74 \pm .47$ & $51.7 \pm .2$ \\
\emph{KF} 
  & $63 \pm .39$ 
  & $74.58 \pm .28$ & $63.37 \pm .53$ & $72.68 \pm .15$ & $67.43 \pm .12$ 
  & $\mathbf{62.22 \pm .38}$ & $55.66 \pm .37$ \\
\hline
\noalign{\vskip 0.5ex}
\multicolumn{3}{l}{\textbf{Condensed Models}} \\
\emph{Averaged KF} 
  & $62.67 \pm .48$ 
  & $70.72 \pm .44$ & $56.4 \pm .8$ & $69.08 \pm .53$ & $64.16 \pm .37$ 
  & $57.08 \pm .43$ & $50.39 \pm .38$ \\
\emph{Distilled KF} 
  & $\mathbf{64.12 \pm .31}$ 
  & $\mathbf{74.96 \pm .33}$ & $\mathbf{63.4 \pm .47}$ & $\mathbf{72.86 \pm .26}$ & $\mathbf{68.42 \pm .23}$ 
  & $62.04 \pm .47$ & $\mathbf{56.28 \pm .24}$ \\
\bottomrule
\emph{improvement} 
  & $9.59 \pm .71$ 
  & $7.51 \pm .5$ & $13.55 \pm .73$ & $7.7 \pm .27$ & $9.39 \pm .65$ 
  & $5.79 \pm .66$ & $6.44 \pm .37$ \\
\bottomrule
\end{tabular}
\end{center}

\footnotesize{\vspace{0.25cm}Mean test accuracy (in percent) and standard error of ResNet18, comparing our method and the baselines on datasets with large label noise and significant overfitting. We include a comparison using the CIFAR-100N dataset, which has innate label noise. The last row shows the improvement of the best performer over the baseline network. Bolded values indicate the best performer in each column.}
\vspace{-0.25cm}
\end{table*}

\begin{table*}[thb!]
  \footnotesize
  \caption{Comparison With Independent 
  Ensembles}
  \label{table:IndependentEnsembleKD}
  
\begin{center}
  \begin{tabular}{l||c|c|c||c|c|c}
    \multicolumn{1}{l||}{\textbf{Method/Dataset}} 
      & \multicolumn{3}{c||}{\textbf{CIFAR-100}} 
      & \multicolumn{3}{c}{\textbf{TinyImageNet}} \\
    
    \multicolumn{1}{r||}{architecture / \% noise} 
      & 0\% & 20\% asym & 40\% asym  
      & 0\% & 20\% sym & 40\% sym \\
      
    \toprule

    \emph{single network} 
      & $78.85 \pm .32$ & $67.07 \pm .42$ & $49.95 \pm .53$ 
      & $65.30 \pm .36$ & $56.25 \pm .47$ & $49.83 \pm .3$ \\

    \hline

    \emph{independent ensemble} 
      & $\mathbf{82.13}$ & $73.89$ & $55.18$ 
      & $\mathbf{71.88}$ & $\mathbf{64.22}$ & $\mathbf{57.56}$ \\

    \emph{distilled independent ensemble} 
      & $79.93 \pm .21$ & $\mathbf{74.54 \pm .18}$ & $\mathbf{56.84 \pm .22}$
      & $67.58 \pm .22$ & $62.54 \pm .25$ & $57.24 \pm .33$ \\
    
    \hline

    \emph{KF ensemble (ours)} 
      & $79.36 \pm .17$ & $74.58 \pm .28$ & $63.37 \pm .53$ 
      & $68.33 \pm .17$ & $\mathbf{62.22 \pm .38}$ & $55.66 \pm .37$ \\

    \emph{distilled KF ensemble (ours)} 
      & $\mathbf{80.29 \pm .16}$ & $\mathbf{74.96 \pm .33}$ & $\mathbf{63.4 \pm .47}$ 
      & $\mathbf{69.96 \pm .15}$ & $62.04 \pm .47$ & $\mathbf{56.28 \pm .24}$ \\

    \emph{averaged KF ensemble (ours)} 
      & $78.48 \pm .26$ & $70.72 \pm .44$ & $56.79 \pm .92$ 
      & $65.17 \pm .23$ & $57.08 \pm .43$ & $50.32 \pm .45$ \\

    \bottomrule
  \end{tabular}
  
\footnotesize{\vspace{0.25cm}A comparison of the Knowledge Fusion ensemble and its two condensed variants with an independent ensemble and its distilled version.}
\end{center}  
 \vspace{-0.25cm}
 
\end{table*}


\subsection{Results}


\myparagraph{Main Results} Table~\ref{table:regularnetworksKD} reports the performance obtained by all variants of our method on CIFAR-100 and TinyImageNet. For comparison, we report the results of both the original model and the aforementioned baselines. Additional results for scenarios associated with overfitting are shown in Table~\ref{table:noisyKD}, where we test our method on these datasets with injected symmetric and asymmetric label noise (see \app\ref{app:implementationdetails}), and CIFAR-100N - a real world dataset with label noise. As is customary, label noise is present only in the training data, while the test data remains clean for evaluation.

The results in Tables~\ref{table:regularnetworksKD} and~\ref{table:noisyKD} reveal a rather surprising finding: the distilled KF model not only matches but often \textbf{outperforms} the KF ensemble of phase 1, on both clean and noisy data. While this may seem counterintuitive—since distillation typically aims to mimic a stronger teacher—it aligns with prior findings that student models can occasionally surpass their teachers in accuracy and robustness~\citep{allen2020towards, pham2022revisiting, das2023understanding}. In our case, this is likely due to the synergy between two known effects: the strength of checkpoint-based ensembles \citep[see][]{wang2022learn, wang2022efficient}, and the regularizing benefits of distillation.

\begin{table*}[!htbp]
\renewcommand{\arraystretch}{1.3}
\caption{Methods Altering the Training Procedure}
\label{table:specialmethods} 
\begin{center}
\footnotesize
\begin{tabular}{l| c || c|c || c|c|c}
\multicolumn{1}{ c |}{Method/\textbf{Dataset}} & \multicolumn{1}{ c ||}{\textbf{CIFAR-100}}& 
\multicolumn{2}{ c ||}{\textbf{CIFAR-100 asym}} & \multicolumn{3}{ c  }{\textbf{CIFAR-100 sym}} \\ 
\multicolumn{1}{ r |}{\% label noise} &     0\% & 20\% & 40\%  &  20\% & 40\%  & 60\% \\
\toprule
\emph{FGE}   & $78.9 \pm .4$& $67.1 \pm .2$& $48.1 \pm .3$& $66.5 \pm .1$& $52.1 \pm .1$ & $38.3 \pm .7$ \\
\emph{SWA}   & $78.8 \pm .1$& $66.6 \pm .1$& $46.9 \pm .2$& $65.6 \pm .4$& $50.0 \pm .1$ & $30.5 \pm .7$\\
\emph{snapshot}   & $78.4 \pm .1$ & $72.1 \pm .4$ & $52.8 \pm .6$& $70.8 \pm .5$ & $63.8 \pm .2$  &  $55.6 \pm .2$ \\
\emph{KF} & $\mathbf{79.36 \pm .17}$& $\mathbf{74.58 \pm .28}$& $\mathbf{63.37 \pm .53}$& $\mathbf{72.68 \pm .15}$ & $\mathbf{67.43 \pm .12}$ & $\mathbf{57.6 \pm .2}$\\
\bottomrule
\end{tabular}

{\footnotesize \vspace{0.5em} Mean test accuracy of Resnet18 on CIFAR100 with and without label noise, comparing KF with baseline methods that alter the training procedure.}
\end{center}
\vspace{-0.25cm}
\end{table*}

\myparagraph{Independent Ensembles}
Since KF yields an ensemble classifier at inference time, we conclude its evaluation by comparing it to a standard independent ensemble and its distilled variant, see Table~\ref{table:IndependentEnsembleKD}. 
We observe that a standard ensemble, which involves training multiple independent networks, often outperforms KF. However, this comes at a substantial increase in training costs, which may render the approach impractical in some scenarios. 

After applying distillation to condense the respective ensembles, this performance gap disappears—a single condensed KF model outperforms the condensed standard ensemble. Surprisingly, in the presence of label noise, the condensed KF can even surpass the original, independent ensemble classifier, achieving the best performance with significantly reduced training and inference complexity. 

After distillation, a consistent pattern emerges from Table~\ref{table:IndependentEnsembleKD}: in most cases, distilling the independent ensemble leads to a drop in performance, whereas the KF-based ensemble improves. This not only narrows the performance gap between the two approaches, but often allows the distilled KF ensemble to outperform its independent counterpart. This pattern is consistent with prior findings on the distillation of dependent versus independent ensembles, see Section~\ref{sec:related-work-distillation}.

\myparagraph{Additional Comparisons}
In Tables~\ref{table:specialmethods} and~\ref{table:additional results} we compare our method to additional methods that adjust the training protocol itself, using both clean and noisy datasets. We employ these methods using the same network architecture as our own, after suitable hyper-parameter tuning.

\begin{table}[htbp]
\footnotesize
\caption{Methods Altering the Training Procedure}
\label{table:additional results}
    
\begin{center}
  \begin{tabular}{l| c||c|c}
    \multicolumn{1}{ c |}{Method/\textbf{Dataset}} & \multicolumn{1}{ c ||}{\textbf{Animal10N}} & \multicolumn{2}{ c }{\textbf{TinyImageNet}} 
    \\ 
    \multicolumn{1}{ r |}{\% label noise} &     8\% & 20\% & 40\% \\
    \toprule
        \emph{FGE}   & $86.5\pm0.6$ & $53.8 \pm .1$&$40.4 \pm .3$ \\
        \emph{SWA}   & $\mathbf{88.1\pm.2}$ & $52.5 \pm .2$&$39.4 \pm .3$ \\
        \emph{snapshot}   & $86.8\pm.3$ & $62.6 \pm .1$ & $56.5 \pm .3$ \\
        \emph{KF}    & $\mathbf{87.8\pm.4}$ & $\mathbf{62.22 \pm .38}$ &$\mathbf{55.6 \pm .39}$\\

  \end{tabular}


\footnotesize{\vspace{0.25cm} Like Fig.~\ref{table:specialmethods}, for additional datasets.}
\end{center} 
  
\vspace{-0.25cm}
\end{table}

\subsection{Ablation Study}
\label{subsec:ablation}

We conducted an extensive ablation study to investigate the limitations, and some practical aspects, of our method, focusing on the key design choices of the KF ensemble (phase 1). For conciseness, some of the studies (\ref{s1} -- \ref{s2} below) are only briefly summarized; full details can be found in the supplementary material, \app\ref{app:additionalablation}.

\begin{table}[!htbp]
\begin{center}
\renewcommand{\arraystretch}{1.3}
\scriptsize
\caption{Comparison Across Architectures on ImageNet}
\label{table:archablation}

\resizebox{\columnwidth}{!}{%

\begin{tabular}{l||c|c|c|c}
\textbf{Method/Architecture} & ResNet50 & ConvNeXt-L & ViT-B/16 & MaxViT-T \\ 
\hline
\emph{Baseline network} & $75.74 \pm .14$ & $82.92 \pm .11$ & $79.16 \pm .1$ & $82.51 \pm .15$ \\
\emph{Horizontal} & $76.42 \pm .1$ & $83.02 \pm .06$ & $79.53 \pm .13$ & $82.93 \pm .14$ \\
\emph{Fixed jumps} & $75.72 \pm .18$ & $83.86 \pm .06$ & $79.11 \pm .13$ & $83.78 \pm .15$ \\
\emph{KF} & $\mathbf{76.52 \pm .16}$ & $\mathbf{83.96 \pm .09}$ & $\mathbf{80.34 \pm .08}$ & $\mathbf{83.81 \pm .14}$ \\
\hline
\emph{Improvement} & $.78 \pm .04$ & $1.03 \pm .13$ & $1.17 \pm .08$ & $1.29 \pm .02$ \\

\end{tabular}

}
\end{center}

{\footnotesize \vspace{0.25cm} Mean test accuracy (in percent) and standard error on ImageNet, comparing our method and baselines across several modern architectures. Bolded values indicate the best performer in each column.}
\end{table}

\subsubsection{Architecture Choice}
To demonstrate the robustness and generality of our method, we evaluate it across a wide range of contemporary neural network architectures. Given the increased representational capacity of these models, we apply them to the ImageNet dataset, which presents a more challenging and large-scale benchmark compared to CIFAR-100 or TinyImageNet. As shown in Table~\ref{table:archablation}, our method consistently outperforms all baselines across all architectures.

\subsubsection{Comparisons to Additional Baselines} 
In addition to the condensing methods presented in Section~\ref{sec:KnowledgeDistillation}, an alternative method to combining checkpoints is to apply an exponential moving average (EMA) to the model weights during training. This technique is known to have some advantages \citep{polyak1992acceleration} and is used sometimes to reduce overfitting \citep{liu2022convnet,dosoViTskiy2020image}, see \citep{tu2022MaxViT} for example. 
In Table~\ref{table:ema}, we explore this option for two datasets while using the same ResNet-18 architecture, showing that our method can be of use when EMA isn't effective, and that EMA improves performance much less than our method.

\begin{table}[t!]
\renewcommand{\arraystretch}{1.3}
\caption{Comparison with EMA}
\label{table:ema}
\begin{center}
\footnotesize
\begin{tabular}{l| c||c}
\multicolumn{1}{ c |}{Method/\textbf{Dataset}} & \multicolumn{1}{ c ||}{\textbf{CIFAR-100}} & \multicolumn{1}{ c }{\textbf{TinyImageNet}}  \\ 
\hline
\emph{EMA}\scriptsize{ (decay = 0.999)} & $-0.34 \pm .14$ & $0.73 \pm .11$ \\
\emph{EMA }\scriptsize{ (decay = 0.9999)} & $-0.06 \pm .33$ & $2.51 \pm .01$ \\
\emph{KF} & $\mathbf{1.05} \pm .14$ & $\mathbf{3.54} \pm .14$ \\
\end{tabular}
\end{center}

{\footnotesize \vspace{0.25cm} Mean (over random validation/test split) improvement in test accuracy (in percent) and standard error on image classification datasets, comparing our method and EMA with different decay values. We use the best epoch for EMA, calculated using the validation set.}
\vspace{-0.25cm}
\end{table}


\begin{table}[h!]
\footnotesize
\begin{center}
\renewcommand{\arraystretch}{1.3}
\caption{Transfer Learning Comparison}
\label{table:transferlearning}
\begin{tabular}{l|| c|c}
\multicolumn{1}{ c ||}{Method/\textbf{Dataset}} & \multicolumn{1}{ c |}{\textbf{CIFAR-100}} & \multicolumn{1}{ c }{\textbf{TinyImageNet}}  \\ 
\hline
\scriptsize{\emph{fully finetuned Resnet18}} & $80.72 \pm .53$ & $75.01 \pm .12$ \\
\scriptsize{\emph{fully finetuned Resnet18 + KF}} & $\mathbf{81.64 \pm .27}$ & $\mathbf{75.6 \pm .18}$ \\
\hline
\scriptsize{\emph{partially finetuned Resnet18}} & $61.7 \pm .60$ & $54.78 \pm .13$ \\
\scriptsize{\emph{partially finetuned Resnet18 + KF}} & $\mathbf{65.24 \pm .67}$ & $\mathbf{59.5 \pm .03}$ \\
\end{tabular}
\end{center}

{\footnotesize \vspace{0.25cm} Mean test accuracy over random validation/test split. Our method uses ResNet18 pre-trained on ImageNet, while finetuning the entire model (top) or only the head (bottom).}
\end{table}

\begin{table*}[thb!]
\begin{center}
\renewcommand{\arraystretch}{1.3}
\footnotesize
\caption{Test Time Augmentation}
\label{table:tta}
\begin{tabular}{l| c || c|c|c || c|c || c}
\multicolumn{1}{ c |}{Method/\textbf{Dataset}} & \multicolumn{1}{ c ||}{\textbf{CIFAR-100}} & \multicolumn{3}{ c ||}{\textbf{CIFAR-100 asym}} & \multicolumn{2}{ c || }{\textbf{CIFAR-100 sym}} & \multicolumn{1}{ c }{\textbf{TinyImageNet}} \\ 
\hline
\multicolumn{1}{ r |}{\% label noise} & 0\% & 10\% & 20\% & 40\%  & 20\% & 40\% & 0\% \\
\toprule
\emph{TTA} & $79.21 \pm .3$ & $72.97 \pm .1$ & $71.00 \pm .1$ & $54.53 \pm .1$ & $70.14 \pm .2$ & $63.59 \pm .1$ & $65.67 \pm .3$ \\
\emph{KF } & $79.36 \pm .17$ & $73.61 \pm .1$ & $71.24 \pm .5$ & $56.19 \pm .8$ & $72.21 \pm .3$ & $65.75 \pm .1$ & $68.33 \pm .17$ \\
\emph{KF + TTA} & $\mathbf{79.55 \pm .3}$ & $\mathbf{74.83 \pm .1}$ & $\mathbf{72.65 \pm .2}$ & $\mathbf{57.71 \pm .3}$ & $\mathbf{72.33 \pm .2}$ & $\mathbf{65.71 \pm .1}$ & $\mathbf{69.05 \pm .3}$ \\
\end{tabular}
\end{center}

{\footnotesize \vspace{0.25cm} Mean test accuracy (in percent) and standard error of ResNet-18, comparing KF with Test Time Augmentation (TTA) on datasets with and without label noise.}
\end{table*}

\subsubsection{Transfer Learning} 
Another popular method to improve performance and reduce overfitting employs transfer learning, in which the model weights are initialized using pre-trained weights over a different task, for example, ImageNet pretrained weights. This is followed by \emph{fine-tuning} either the entire model or only some of its layers (its head for example). In Table~\ref{table:transferlearning}, we show that our method is complementary to transfer learning, improving performance in this scenario as well. Notably, when fine-tuning only the final layers, the overhead of KF is \textbf{negligible}, as the majority of the model is used only once at inference. The only overhead comes from the memory and inference costs of storing and processing the different checkpoint heads used during fine-tuning, which are small in comparison to the size of the full model.

\subsubsection{Test Time Augmentation}
In TTA, each test sample is classified multiple times using different augmentations, and the final prediction is obtained by averaging the class probabilities from all augmented versions. The results in Table~\ref{table:tta} show that our method is comparable or better than TTA, even in the presence of label noise. Moreover, our approach is complementary to TTA, offering potential benefits when used together.

\subsubsection{Model Size} A common practice nowadays is to use very large neural networks, with hundreds of millions of parameters, or even more. However, enlarging models does not always improve performance, as a large number of parameters can lead to overfitting. In Fig.~\ref{subfig:modelsizeforget} we show that indeed the forget fraction increases with the model's size, which also improves the benefit of our approach (see Table~\ref{table:modelsize}), making it especially useful while using large models.

\begin{table}[thb!]
\begin{center}
\renewcommand{\arraystretch}{1.3}
\footnotesize

\caption{Comparison of Model Size}
\label{table:modelsize}
\begin{tabular}{l|| c|c|c}
\multicolumn{1}{ c ||}{Method/\textbf{model size}} & \multicolumn{1}{ c |}{\textbf{small}} & \multicolumn{1}{ c | }{\textbf{base}} & \multicolumn{1}{ c }{\textbf{large}}  \\ 
\hline
\emph{single network}   & $83.21 \pm .01$ & $83.31 \pm .15$ & $82.92 \pm .09$ \\
\emph{KF}   & $83.17 \pm .04$ & $\mathbf{83.57 \pm .15}$ & $\mathbf{83.96 \pm .09}$ \\
\end{tabular}
\end{center}

{\footnotesize \vspace{0.25cm} Mean (over random validation/test split) test accuracy (in percent) and standard error on image classification datasets, comparing our method and the original predictor (ConvNeXt, trained on ImageNet) with a varying number of parameters.}
\vspace{-0.25cm}
\end{table}

\subsubsection{Using Training Set for Validation}
\label{s1}
The results reported above summarize experiments in which half of the \emph{test data} was used for validation, to evaluate our method's hyper-parameters (the list of alternative epochs and $\{\varepsilon_i\}$). The accuracy reported above was computed over the remaining test data, averaged over three random splits of validation and test data with different random seeds. In our ablation study we repeated these experiments using the original train/test split, where a subset of the training data is set aside for hyper-parameter tuning.  As customary, these same parameters are later used with models trained on the full training set, demonstratively without deteriorating the results.

\subsubsection{Number of Checkpoints}
Knowledge Fusion remains effective when using only a small fraction (5–10\%) of the available training checkpoints.

\subsubsection{Optimal vs Sub-Optimal Training}
KF is particularly useful under limited hyperparameter tuning; in sub-optimal training regimes, it helps close the gap to optimally trained models.

\subsubsection{Fairness}
KF does not negatively affect model fairness across standard fairness benchmarks.

\subsubsection{Window Size}
\label{s2}
We find that using a window size of $w=1$ when averaging checkpoints is both necessary and near-optimal for the effectiveness of the method.



\subsection{Summary and Discussion}  
\label{subsec:KKFDdiscussion}

Our method significantly improves performance in modern neural networks. It complements other overfitting reduction methods like EMA and proves effective where these methods fail, as in fine-tuning of pre-trained models. In challenging scenarios, such as small networks handling complex data or datasets with label noise, it further enhances performance, reducing errors by around 15\% in cases of 10\% asymmetric noise. Our approach outperforms or matches baselines, especially in settings like ViT16 over ImageNet and Resnet18 over TinyImageNet, regardless of training choices. Unlike some horizontal methods and fixed-jump schedules that show limited improvement, our method remains effective without extensive hyper-parameter tuning. 

We provide a spectrum of model variants derived from the core KF algorithm, offering different trade-offs between performance, training cost, and inference efficiency. While the distilled model consistently yields the strongest results, simple weight averaging provides clear improvements over the vanilla model, without requiring an additional training phase or the maintenance of multiple networks at inference time. This flexibility allows practitioners to choose the variant that best aligns with their computational and deployment constraints.



\section{Forgotten Knowledge: Theoretical Angle} 
\label{sec:TheoryForgottenKnowledge}

To gain insight into the nature of knowledge forgotten while training a deep model with Gradient Descent (GD), we analyze over-parameterized deep linear networks trained by GD. These models are constructed through the concatenation of linear operators in a multi-class classification scenario: $\by = W_L \cdot \ldots \cdot W_1 \bx$, where $\bx \in \R^d$.  
For simplicity, we focus on the binary case with two classes, with the following objective function:
\begin{equation}
\label{eq:object}
\min_{W_1,\ldots,W_L}\sum_{i=1}^n \Vert W_L\cdot\ldots \cdot W_1 \bx_i - y_i \Vert^2  
\end{equation}
Above the matrices $\{W_l\}_{l=1}^L$ represent the $2D$ matrices corresponding to $L$ layers of a deep linear network, and points $\{\bx_i\}_{i=1}^n$ represent the training set with labeling function $y_i=\pm 1$ for the first and second classes, respectively. Note that $\bW = \prod_{l=L}^1 W_l$ is a row vector that defines the resulting separator between the classes. The classifier is defined as:
$f(\bx) = {\text{sign}}~ \left (\prod_{l=L}^1 W_l \bx \right )$ 
for $\bx\in\R^d$.

\myparagraph{Preliminaries.}
Let $\bW^{(n)} = \prod_{l=L}^1 W_l^{(n)}$ denote the separator after $n$ GD steps, where $\bW^{(n)} \equiv [w_1^{(n)}, \ldots, w_d^{(n)}] \in \R^d$. For convenience, we rotate the data representation so that its axes align with the eigenvectors of the data's covariance matrix. \citep{hacohen2022principal} showed that the convergence rate of the $j^\mathrm{th}$ element of $\bW$ with respect to $n$ is exponential, governed by the corresponding $j^\mathrm{th}$ eigenvalue:
\begin{equation}
\label{eq:linear-comb}
w_j^{(n)}\approx \lambda_j^{n}w_j^{(0)} + [1- \lambda_j^{n}]w_j^{opt}, \qquad \lambda_j = 1-\gamma s_j L
\end{equation}
Here, $\bW^{(0)}$ denotes the separator at initialization, $\bw^{opt}$ denotes the optimal separator (which can be derived analytically from the objective function), $s_j$ represents the $j^\mathrm{th}$ singular value of the data, and $\gamma$ is the learning rate. Notably, while $\bw^{opt}$ is unique and $\prod_{l=L}^1 W_l^{(n)}\xrightarrow[n \to \infty]{} \bw^{opt}$, the specific solution at convergence $\{W_l^{(\infty)}\}_{l=1}^L$ is not unique.

\myparagraph{Forget Time in Deep Linear Models.}
Let $\Lambda$ denote $\text{diag}(\{\lambda_j\})$ - a diagonal matrix in $\R^{d\times d}$, and I the identity matrix. It follows from (\ref{eq:linear-comb}) that
\begin{equation}
\label{eq:linear-comb-mat}
\bW^{(n)}\approx \bW^{(0)}\Lambda^{n} + \bW^{opt}[I- \Lambda^{n}]
\end{equation}

We say that a point is forgotten if it is classified correctly at initialization, but not so at the end of training. Let $\bx$ denote a forgotten datapoint, and let $N$ denote the number of GD steps at the end of training. Since by definition $f(\bx) = {\text{sign}} (\bW^{(n)} \bx)$, it follows that $\bx$ is forgotten iff 
$\{\bW^{(0)} y\bx > 0\}$ and $\{\bW^{(N)} y\bx < 0\}$. 

Let us define the forget time of point $\bx$ as follows:
\begin{defn}[Forget time]
\label{def:1}
GD iteration $\hat n$ that satisfies 
\begin{equation}
\label{eq:forget-time}
\begin{split}
&\bW^{(\hat n)} y\bx \leq 0 \\
& \bW^{(n)} y\bx > 0 \qquad \forall n<\hat n
\end{split}
\end{equation}
\end{defn}
\begin{claim}
\label{eq:forgt-time}
Each forgotten point has a finite forget time $\hat n$.
\end{claim}
\begin{proof}
Since $\{\bW^{(0)} y\bx > 0\}$ and $\{\bW^{(N)} y\bx < 0\}$, (\ref{eq:forget-time}) follows by induction.
\end{proof}

Note that Def~\ref{def:1} corresponds with the \emph{Forget time} seen in deep networks (cf. Fig.~\ref{subfig:forgettimetimg}). The empirical investigation of this correspondence is discussed in \citep{stern2025overfit}.


To characterize the time at which a point is forgotten, we inspect the rate at which $F (n)=\bW^{(n)} y\bx $ changes with $n$. We begin by assuming that the learning rate $\gamma$ is infinitesimal, so that terms of magnitude $O(\gamma^2)$ can be neglected. Using (\ref{eq:linear-comb-mat}) and the Taylor expansion of $\lambda_j$ from (\ref{eq:linear-comb})
\begin{equation*}
\begin{split}
F(n) \approx & \left ( \bW^{(0)} - \bW^{opt}  \right ) \Lambda^{n} y \bx + \bW^{opt} y \bx\\
=&  ~\bW^{opt} y \bx + \sum_{j=1}^d ( w_j^{(0)} - w_j^{opt}  ) \lambda_j^{n} y x_j \\
=&  ~\bW^{opt} y \bx + \sum_{j=1}^d ( w_j^{(0)} - w_j^{opt}  ) [\mathsmaller{\mathsmaller{1-n\gamma s_j L + O(\gamma^2) }}] y x_j \\
=&  ~ \bW^{(0)} y \bx - n  \gamma L\sum_{j=1}^d ( w_j^{(0)} - w_j^{opt} ) y s_j x_j + O(\gamma^2) 
\end{split}
\end{equation*}
It follows that
\begin{equation}
\label{eq:forget-time-diff}
\frac{d F(n) }{d n} = -\gamma  y L\sum_{j=1}^d ( w_j^{(0)} - w_j^{opt} ) s_j x_j + O(\gamma^2)
\end{equation}

\noindent
\textbf{Discussion.} 
Recall that $\{s_j\}$ is the set of singular values, ordered such that $s_1 \geq s_2 \geq \dots \geq s_d$, and $x_j$ is the projection of point $\bx$ onto the $j^{\text{th}}$ eigenvector. From (\ref{eq:forget-time-diff}), the rate at which a point is forgotten, if at all, depends on vector $[s_j x_j]_j$, in addition to the random vector $\bW^{(0)} - \bW^{\text{opt}}$ and label $y$. All else being equal, a point will be forgotten faster if the length of its spectral decomposition vector $[x_j]$ is dominated by its first  components, indicating that most of its mass is concentrated in the leading principal components.


\section{Summary and Conclusions}

We revisited the problem of \emph{overfitting} in deep learning by proposing a method to track the forgetting of validation data as a means to detect local overfitting. We linked this perspective to the phenomenon of \emph{epoch-wise double descent}, empirically extending its scope and showing that a similar effect arises even in benchmark datasets with clean labels. Motivated by these empirical insights, we introduced a simple yet general method to improve classification performance at inference time. We then demonstrated its effectiveness across a range of datasets and modern network architectures. These results confirm that models do forget useful information in the later stages of training, and provide a proof of concept that recovering this knowledge can lead to improved performance. When combining this method with Knowledge Distillation, we showed enhanced performance accompanied by reduced complexity - particularly in noisy environments - offering a win-win strategy that improves accuracy while reducing both training and inference costs.

\subsection*{\textbf{Acknowledgments}}
This work was supported by grants from the Israeli Council of Higher Education, AFOSR award FA8655-24-1-7006, and the Gatsby Charitable Foundation.

\bibliographystyle{IEEEtran}
\bibliography{bib}

\newpage

\begin{figure*}[htbp]
    \centering
    \subfloat[\footnotesize ImageNet, ResNet50, steplr\label{subfig:r18forgetImg}]{
        \includegraphics[width=0.26\linewidth]{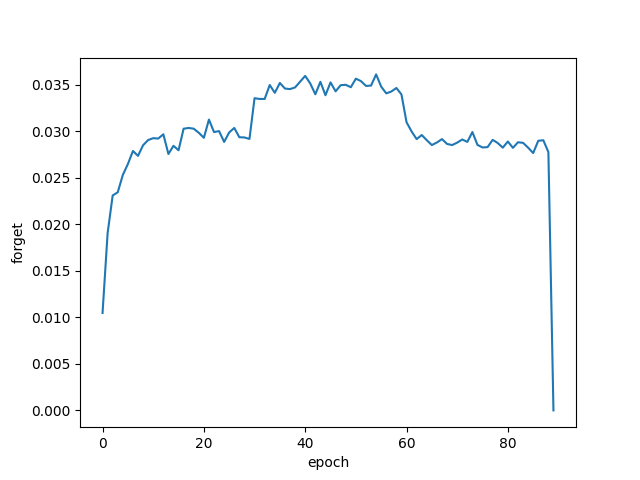}
    }\hspace{-0.25cm}
    \subfloat[\footnotesize Cifar100, DenseNet121\label{subfig:ForgetDenseC100}]{
        \includegraphics[width=0.235\linewidth]{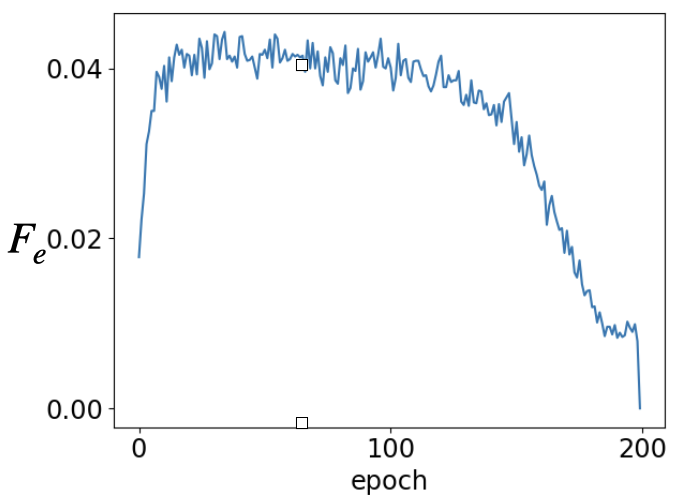}
    }\hspace{-0.25cm}
    \subfloat[\footnotesize TinyImageNet, DenseNet121\label{subfig:Denseforgettiny}]{
        \includegraphics[width=0.235\linewidth]{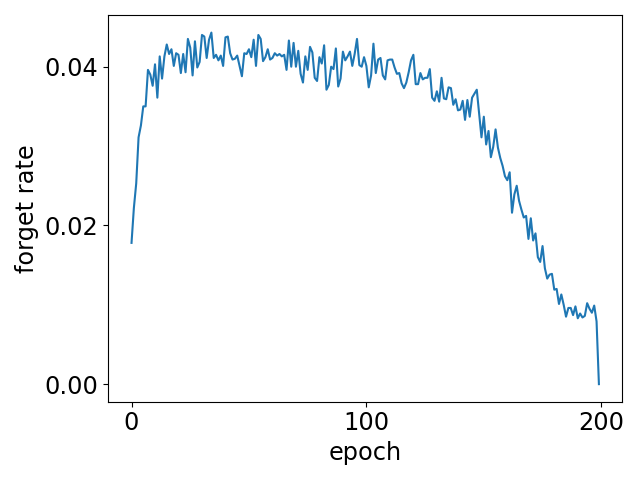}
    }\hspace{-0.25cm}
    \subfloat[\footnotesize TinyImageNet, ResNet18\label{subfig:r18forgetTimg}]{
        \includegraphics[width=0.235\linewidth]{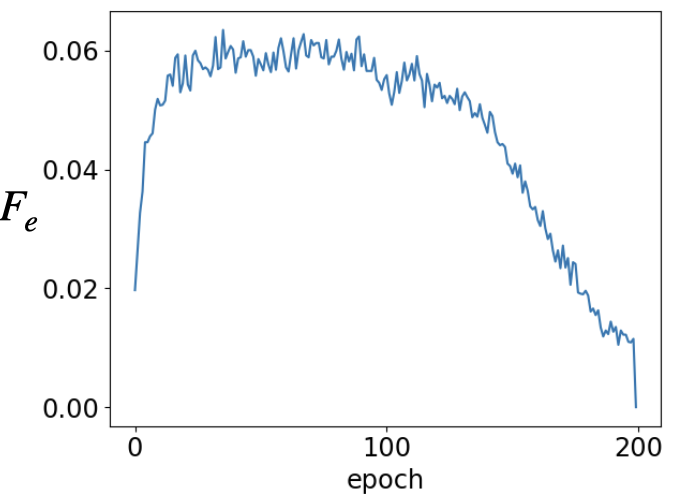}
    }
    \caption[forget]{The forget fraction $F_e$, as defined in (\ref{eq:forget}), of common neural networks trained on various image classification datasets and different architectures.}
    \label{fig:moreforgetrate}
\end{figure*}

\begin{figure*}[h!]
    \centering
     \subfloat[\label{subfig:trainsizeforget}]{
        \includegraphics[width=0.28\linewidth]{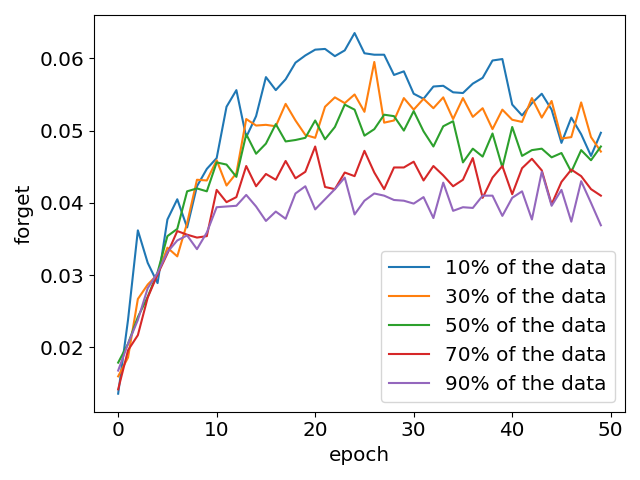}
    }\hspace{.2cm}
     \subfloat[\label{subfig:cifar100sumcorrect}]{
        \includegraphics[width=0.28\linewidth]{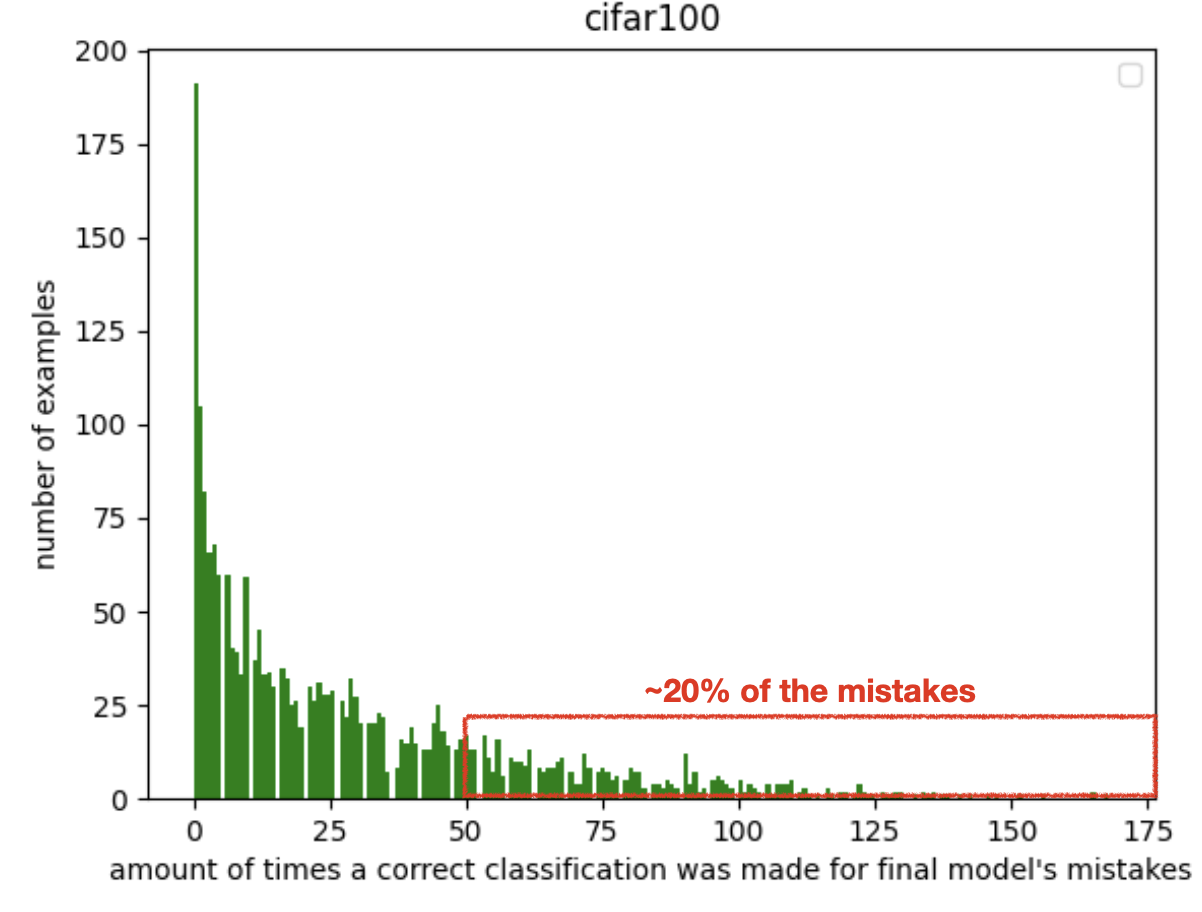}
    }\hspace{.2cm}
     \subfloat[\label{subfig:forgettimecifar100}]{
        \includegraphics[width=0.28\linewidth]{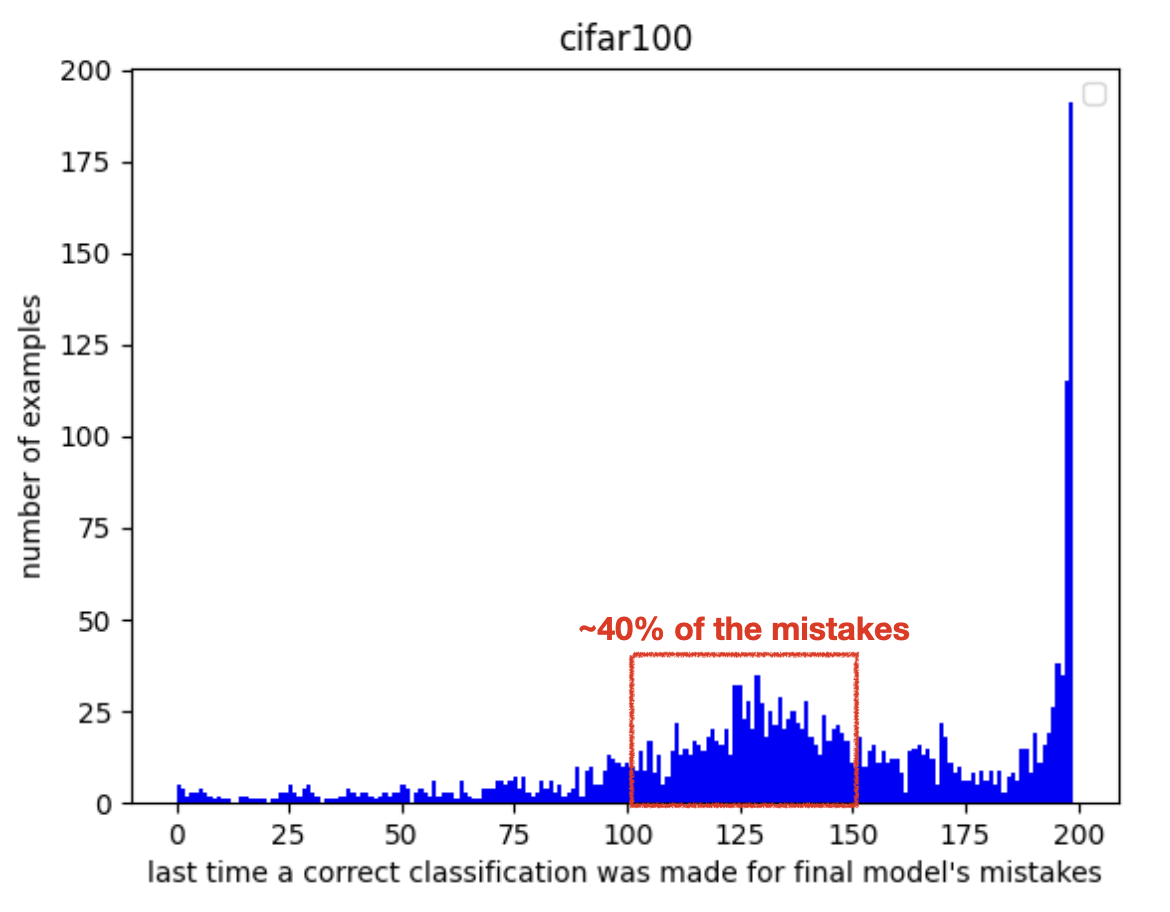}
    }
   \caption[forget]{CIFAR100, ResNet18: (a) The $F_e$ score of ResNet18 trained on $10/30/50/70/90\%$ of the train data in CIFAR-100 (purple/red/green/yellow/blue line, respectively) in the first 50 epochs of training (after which the score decreases); $F_e$ is significantly larger for the smallest set of only $10\%$. (b-c) Within the set of wrongly classified test points after training, we show (b) the fraction that was correctly predicted (y-axis) for x epochs (x-axis), and (c) the last epoch in which an example was classified correctly.}
    \label{fig:forgetrate}
\end{figure*}

\newpage

\appendices
\setcounter{page}{1}

\section{Additional Demonstrations of Forgetting}
\label{app:moreforget}

We first show more examples of various neural networks trained on different datasets which show significant forgetting during training (Fig.~\ref{fig:moreforgetrate}), in order to further demonstrate the generality of this phenomenon. dditional results on local overfitting with different datasets are shown in Fig.~\ref{fig:forgetrate}.

\section{Full Implementation Details}
\label{app:implementationdetails}

We conducted experiments on four image classification datasets: ImageNet \citep{deng2009imagenet}, CIFAR-100 \citep{krizhevsky2009learning}, TinyImageNet \citep{le2015tiny}, and CIFAR-100N \cite{wei2022learning}. For ImageNet, we trained ResNet \citep{he2016deep}, ConvNeXt \citep{liu2022convnet}, ViT \citep{dosoViTskiy2020image}, and MaxViT \citep{tu2022MaxViT} using the torchvision training code\footnote{https://github.com/pytorch/vision/tree/main/references/classification} with recommended hyperparameters, except ConvNeXt, where we followed its official implementation\footnote{https://github.com/facebookresearch/ConvNeXt} without EMA (see Section~\ref{subsec:ablation} for comparison to EMA). For transfer learning, we resized images to $224 \times 224$, used a 0.001 learning rate, and initialized with ImageNet weights. We either finetuned the full network or trained a new two-layer head, with the first layer outputting 100× the embedding size.

On CIFAR-100 and TinyImageNet without label noise, all models were trained for 200 epochs with batch size 32, learning rate 0.01, SGD (momentum 0.9, weight decay 5e-4), cosine annealing, and standard augmentations (horizontal flip, random crop). For noisy label experiments, we used cosine annealing with warm restarts (every 40 epochs), a 0.1 learning rate updated per batch, and larger batch sizes: 64 for CIFAR-100 and 128 for TinyImageNet. 

To evaluate the effectiveness of combining Knowledge Fusion with Knowledge Distillation, several experiments were conducted including a variety of KD hyperparameters. While the results presented above in Tables~\ref{table:regularnetworksKD} and~\ref{table:noisyKD}) use the most promising configuration (\( T = 2.5 \) and \( \alpha = 0.9 \)), we note that they are robust—similar performance is observed across a wide range of hyperparameter values.

To ensure fair comparison, we trained all methods in Table~\ref{table:specialmethods} from scratch using our architecture and data. For \citep{huang2017snapshot}, we followed the original protocol; for \citep{izmailov2018averaging, garipov2018loss}, we applied our training scheme -- since these methods augment existing training -- and tuned hyperparameters for optimal performance. Experiments were run on A5000 GPUs.

\begin{table*}[htbp]
\begin{center}
\footnotesize
\caption{}
\label{table:fairness}
  \centering
  \begin{tabular}{l| c|c|c||c|c|c}
    \multicolumn{1}{ c |}{Dataset/\textbf{Method}} & \multicolumn{3}{ c ||}{\textbf{original model}} & \multicolumn{3}{ c }{\textbf{KF}} 
    \\ 
        \hline
    \multicolumn{1}{ r |}{evaluation metric} & natural accuracy&  transformed accuracy & bias & natural accuracy& transformed accuracy & bias \\
    \hline
        \emph{w/o color}   & $89.07 \pm .48$ & $87.98 \pm .38$ & $0.07 \pm .001$ & $89.90 \pm .40$ & $87.85 \pm .48$ & $0.07 \pm .002$ \\
        \emph{Center cropped to 28x28}   & $88.45 \pm .31$ & $70.44 \pm .44$ & $0.13 \pm .003$ & $88.92 \pm .32$ & $70.21 \pm .74$ & $0.13 \pm .004$\\
        \emph{Downsampled to 16x16}   & $85.43 \pm .32$ & $76.70 \pm .13$ & $0.08  \pm .001$ & $86.55 \pm .27$ & $77.47 \pm 14$ & $0.07  \pm .001$\\
        \emph{Downsampled to 8x8}   & $80.061 \pm .33$ & $52.03 \pm .49$ & $0.22 \pm .002$ & $81.48 \pm .44$ & $52.99 \pm .49$ & $0.21 \pm .003$\\
        \emph{With ImageNet replacements}   & $88.45 \pm .31$ & $70.44  \pm .44$ & $0.13 \pm .003$ & $88.92 \pm .32$ & $70.44 \pm .44$ & $0.13 \pm .004$\\
  \end{tabular}
  
  \vspace{0.2cm}
\end{center}
\footnotesize{CIFAR10: Mean (over random validation/test split) test accuracy and amplification bias (in percent) and standard error on natural and transformed test sets, comparing our method and the original model.} 
  
\end{table*}

\subsection*{Injecting label noise} 
For label noise experiments, we injected noise using two standard methods \citep{patrini2017making}:
\begin{enumerate}
\item \textbf{Symmetric noise:} a fraction $p \in {0.2, 0.4, 0.6}$ of labels is randomly selected and replaced uniformly with a different label.
\item \textbf{Asymmetric noise:} a fraction $p$ of labels is randomly selected and altered using a fixed label permutation.
\end{enumerate}

  
  

\section{Ablation Study}
\label{app:additionalablation}

\subsection{Using Training Set for Validation}
\label{abl:val}

In this experiment, we follow a common practice with respect to the validation data: we train our model on CIFAR-100 and TinyImageNet using only 90\% of the train data, use the remaining 10\% for validation, and finally retrain the model on the full train data while keeping the same hyper-parameters for inference. The results are almost identical to those reported in Table~\ref{table:regularnetworksKD}. This validates the robustness of our method to the (lack of) a separate validation set. 

\subsection{Number of Checkpoints}
\label{subsec:checkpointsnum}

Here we evaluate the cost entailed by the use of an ensemble at inference time. In Fig.~\ref{fig:ensemblesize} we report the improvement in test accuracy as compared to a single network, when varying the ensemble size. The results indicate that almost all of the improvement can be obtained using only $5-10\%$ of the checkpoints, making our method practical in real life.

\begin{figure}[htbp]
\vspace{-.5cm}
    \centering
    \subfloat[\resizebox{.4\linewidth}{!}{CIFAR-100}\label{subfig:cifar100enssize}]{
        \includegraphics[width=0.22\linewidth]{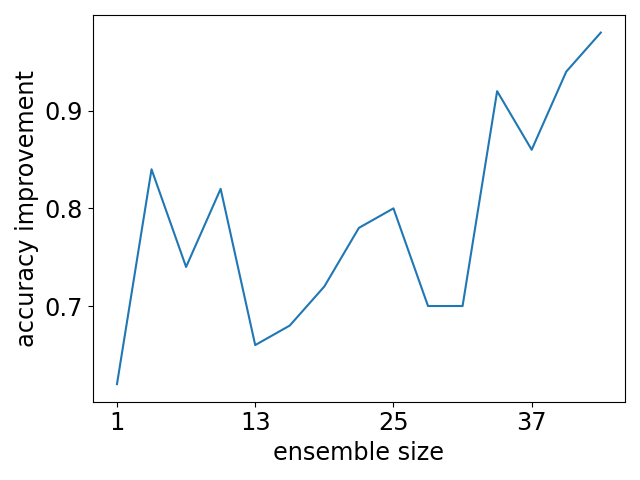}
    }\hfill
    \subfloat[\resizebox{.6\linewidth}{!}{TinyImageNet}\label{subfig:tinyenssize}]{
        \includegraphics[width=0.22\linewidth]{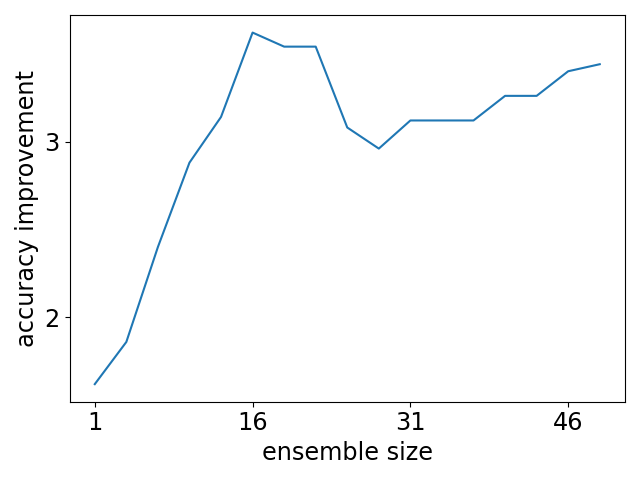}
    }\hfill
    \subfloat[\resizebox{.75\linewidth}{!}{20\% asym noise}\label{subfig:c100asym20enssize}]{
        \includegraphics[width=0.22\linewidth]{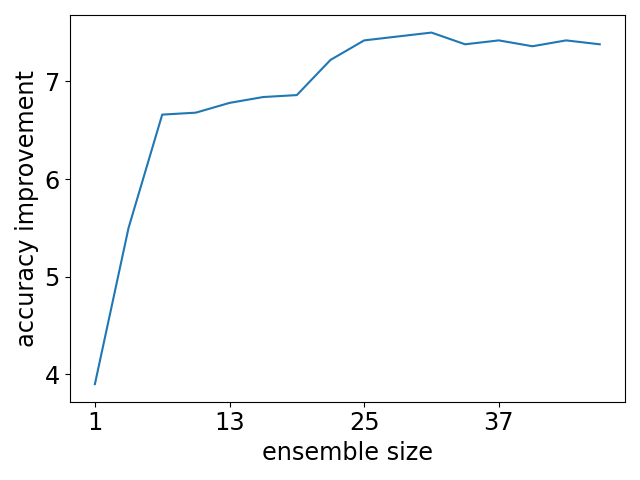}
    }\hfill
    \subfloat[\resizebox{.7\linewidth}{!}{20\% sym noise}\label{subfig:tinyenssym20size}]{
        \includegraphics[width=0.22\linewidth]{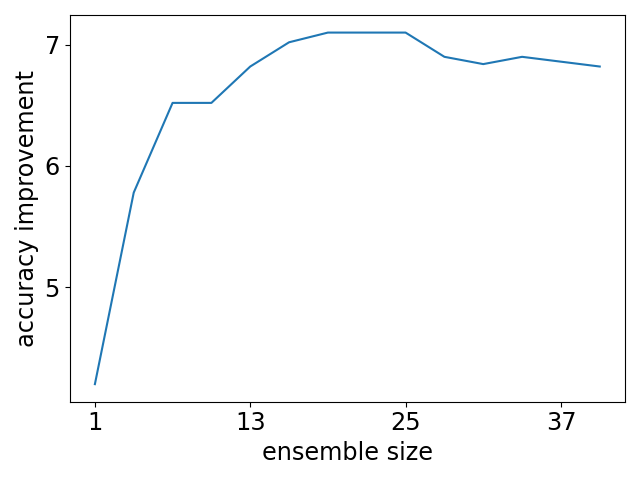}
    }

    \caption[ensemble size]{Improvement achieved by our method when using a different number of checkpoints (shown on the x-axis).}
    \label{fig:ensemblesize}
    \vspace{-0.5cm}
\end{figure}

\subsection{Optimal vs Sub-Optimal Training}
\label{abl:subopttrain}

In real life, full search for the optimal training scheme and hyper-parameters is not always possible, leading to sub-optimal performance. Interestingly, our method can be used to reduce the gap between optimal and sub-optimal training sessions, as seen in Table~\ref{table:suboptimal}, where this gap when training MaxViT over ImageNet is reduced by almost half when applying our method (on both models). 

\begin{table}[htbp]
\footnotesize
\caption{}
\label{table:suboptimal}
\begin{center}
\centering
  \begin{tabular}{l| c||c}
    \multicolumn{1}{ c |}{training/\textbf{method}} & \multicolumn{1}{ c ||}{\textbf{original network}} & \multicolumn{1}{ c }{\textbf{KF}} 
    \\ 
    \toprule
    \emph{regular training}   & $82.5 \pm .1$ & $83.8 \pm .1$ \\
    \emph{sub-optimal training}   & $77.3 \pm .1$ & $81.0 \pm .1$\\
    \hline
    \emph{improvement}   & $\mathbf{5.2}$ &$\mathbf{2.8}$\\

  \end{tabular}

\vspace{0.2cm}
\end{center}
\footnotesize{Mean test accuracy (in percent) and ste, over random validation/test split. MaxViT is trained to classify ImageNet, comparing optimal and sub-optimal training with and without KF.} 
    \vspace{-0.25cm}  

\end{table}

\subsection{Fairness}
\label{abl:fairness}

In this section we study our method's effect on the model's \emph{fairness}, i.e. the effect non-relevant features have on the classification of test data examples. We follow \citep{wang2020towards} and train and test our models on datasets in which they might learn spurious correlations. To create those datasets, we divide the classes into two groups: in each class of the first group 95\% of the training images goes through a transformation (and the rest remain unchanged), and vice versa for the classes of the second group. The transformations we use are: removing color, lowering the images resolution (by down sampling and up sampling), and replacing images with downsampled images for the same class in ImageNet. We use cifar10 in our evaluation as done in  \citep{wang2020towards}, and use also CIFAR-100 with the remove color transformation (the rest of the transformations were less appropriate for this datasets, as it contains similar classes that could actually become harder to seperate at a lower resolution). We use the same method as before for our validation data, and thus the validation is of the same distribution as the test data.

Our evaluation uses the following metrics: (i) the test accuracy on two test sets (with/without the transformation), which should be lower if the model learns more spurious correlations, and (ii) the amplification bias defined in \citep{zhao2017men}, which is defined as follows:

\begin{equation}
    \frac{1}{|C|}\sum_{c \in C} \frac{max(c_T, c_N)}{c_T+c_N} - 0.5
\end{equation}

When $C$ is the group of classes, $c_T$ is the number of images from the transformed test set predicted to be of class c, and $c_N$ is the number of images from the natural test set predicted to be of class c - we would like those to be as close as possible, since the transformation shouldn't change the prediction, and thus the lower the score the better. 

The results of our evaluation are presented in Table~\ref{table:fairness}. Clearly, our method improves the average performance on both datasets without deteriorating the amplification bias, which indicates that our method has no negative effects on the model's fairness.

\subsection{Window Size}
\label{abl:window}

\begin{figure}[htbp]
    \centering
    \subfloat[\resizebox{.6\linewidth}{!}{40\% sym noise}\label{subfig:0.4symnoiseC100}]{
        \includegraphics[width=0.3\linewidth]{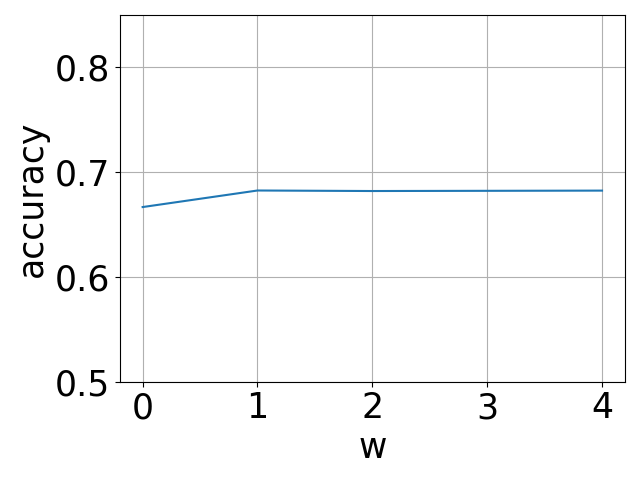}
    }\hfill
    \subfloat[\resizebox{.6\linewidth}{!}{20\% asym noise}\label{subfig:0.2asymnoiseC100}]{
        \includegraphics[width=0.3\linewidth]{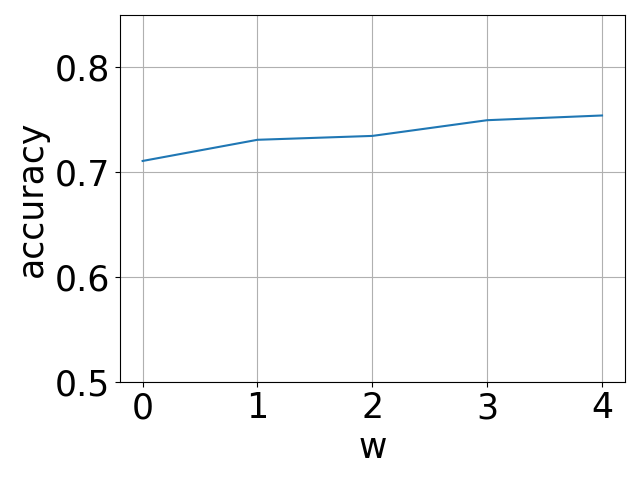}
    }\hfill
    \subfloat[\resizebox{.6\linewidth}{!}{40\% asym noise}\label{subfig:0.4asymnoiseC100}]{
        \includegraphics[width=0.3\linewidth]{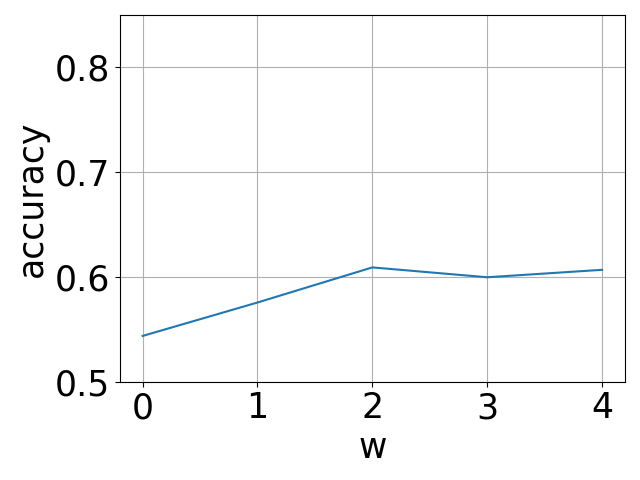}
    }

    \caption[windowsize]{Accuracy of the Knowledge Fusion method as a function of window size ($w$), evaluated on the CIFAR-100 dataset under varying noise levels. The results demonstrate that, irrespective of the noise level, the algorithm achieves near-optimal performance for $w\geq1$.} 
    \label{fig:windowsize}
\end{figure}

The proposed algorithm introduces a question regarding the selection of the window size ($w$) around the chosen epochs. In this work, we present results that illustrate the benefit of employing such window. Our findings, shown in Fig.~\ref{fig:windowsize}, demonstrate that in all evaluated scenarios, using a window greater than zero consistently outperforms the absence of a window.

\end{document}